%% file: main.tex
\documentclass{article}

\usepackage{microtype}
\usepackage{graphicx}
\usepackage{subfigure}
\usepackage{booktabs} % for professional tables

\usepackage{hyperref}

\usepackage[accepted]{icml2025}
\usepackage{amsmath}
\usepackage{amssymb}
\usepackage{mathtools}
\usepackage{amsthm}

\usepackage[capitalize,noabbrev]{cleveref}

\theoremstyle{plain}
\newtheorem{theorem}{Theorem}[section]

\newtheorem{lemma}[theorem]{Lemma}

\theoremstyle{definition}
\newtheorem{definition}[theorem]{Definition}

\theoremstyle{remark}

\input{definitions.tex}

\usepackage[textsize=tiny]{todonotes}

\icmltitlerunning{Policy Gradient with Tree Expansion}

\begin{document}

\twocolumn[
\icmltitle{Policy Gradient with Tree Expansion}

% It is OKAY to include author information, even for blind
% submissions: the style file will automatically remove it for you
% unless you've provided the [accepted] option to the icml2025
% package.

% List of affiliations: The first argument should be a (short)
% identifier you will use later to specify author affiliations
% Academic affiliations should list Department, University, City, Region, Country
% Industry affiliations should list Company, City, Region, Country

% You can specify symbols, otherwise they are numbered in order.
% Ideally, you should not use this facility. Affiliations will be numbered
% in order of appearance and this is the preferred way.
\icmlsetsymbol{equal}{*}

\begin{icmlauthorlist}
\icmlauthor{Gal Dalal}{equal,yyy}
\icmlauthor{Assaf Hallak}{equal,yyy}
\icmlauthor{Gugan Thoppe}{iis}
\icmlauthor{Shie Mannor}{yyy,tch}
\icmlauthor{Gal Chechik}{yyy,biu}
%\icmlauthor{}{sch}
%\icmlauthor{}{sch}
\end{icmlauthorlist}

\icmlaffiliation{yyy}{NVIDIA Research}
\icmlaffiliation{iis}{Indian Institute of Science}
\icmlaffiliation{tch}{Technion University}
\icmlaffiliation{biu}{Bar-Ilan University}

\icmlcorrespondingauthor{Gal dalal}{gdalal@nvidia.com}
\icmlcorrespondingauthor{Assaf Hallak}{ahallak@nvidia.com}

% You may provide any keywords that you
% find helpful for describing your paper; these are used to populate
% the "keywords" metadata in the PDF but will not be shown in the document
\icmlkeywords{Machine Learning, ICML}

\vskip 0.3in
]

% this must go after the closing bracket ] following \twocolumn[ ...

% This command actually creates the footnote in the first column
% listing the affiliations and the copyright notice.
% The command takes one argument, which is text to display at the start of the footnote.
% The \icmlEqualContribution command is standard text for equal contribution.
% Remove it (just {}) if you do not need this facility.

%\printAffiliationsAndNotice{}  % leave blank if no need to mention equal contribution
\printAffiliationsAndNotice{\icmlEqualContribution} % otherwise use the standard text.

\begin{abstract}
Policy gradient methods are notorious for having a large variance and high sample complexity. To mitigate this, we introduce \treepol{}---a generalization of softmax that employs planning. In \treepol{}, we extend the traditional logits with the multi-step discounted cumulative reward, topped with the logits of future states. We analyze \treepol{} and explain how tree expansion helps to reduce its gradient variance. We prove that the variance depends on the chosen tree-expansion policy. Specifically, we show that the closer the induced transitions are to being state-independent, the stronger the variance decay. With approximate forward models, we prove that the resulting gradient bias diminishes with the approximation error while retaining the same variance reduction. Ours is the first result to bound the gradient bias for an approximate model. In a practical implementation of \treepol{}, we utilize a parallel GPU-based simulator for fast and efficient tree expansion. Using this implementation in Atari, we show that \treepol{} reduces the gradient variance by three orders of magnitude. This leads to better sample complexity and improved performance compared to distributed PPO.
\end{abstract}

\section{Introduction}

Policy Gradient (PG) 
 methods \citep{sutton1999policy}
 for Reinforcement Learning (RL) are often the first choice for environments that allow numerous interactions at a fast pace \citep{schulman2017proximal}. Their success is attributed to several factors: they are easy to distribute to multiple workers, require no assumptions on the underlying value function, and have both on-policy and off-policy variants. 

 Despite these positive features, 
PG algorithms are also notoriously unstable due to the high variance of the gradients computed over entire trajectories \citep{liu2020improved, xu2020improved}. As a result, PG algorithms tend to be highly inefficient in terms of sample complexity. Several solutions have been
%were 
proposed to mitigate the high variance issue,
including baseline subtraction \citep{greensmith2004variance,  thomas2017policy, wu2018variance}, anchor-point averaging \citep{papini2018stochastic}, and other variance reduction techniques \citep{zhang2021convergence, shen2019hessian, pham2020hybrid}.

A second family of algorithms that achieved state-of-the-art results in several domains is based on planning. Planning is exercised primarily in the context of value-based RL and is usually implemented using a Tree Search (TS) \citep{ silver2016mastering, schrittwieser2020mastering}. 
In this work, we
combine PG with TS by introducing a parameterized differentiable policy that incorporates tree expansion. Namely, our \treepol{} policy replaces the standard policy logits of a state and action, with the expected value of trajectories that originate from these state and action.  We consider two variants of SoftTreeMax, one for cumulative reward and one for exponentiated reward.

Unlike approaches that incorporate multi-step returns for value function estimation (e.g., n-step TD methods), our work explicitly integrates planning into the policy parameterization itself. This distinction is crucial—while n-step returns serve as a Monte Carlo estimation technique for the advantage function $A$ in the gradient estimator $E[(\nabla \log \pi)A]$, our SoftTreeMax affects the policy $\pi$ directly. This enables us to obtain fundamentally different variance reduction properties while keeping the policy gradient framework intact.

Combining TS and PG should be done with care given the biggest downside of PG---its high gradient variance. This raises   questions that were ignored until this work: (i) How to design a PG method based on tree-expansion that is stable and performs well in practice? and (ii) How does the tree-expansion policy affect the PG variance?  
Here, we analyze \treepol, and provide a practical methodology to choose the expansion policy to minimize the resulting variance. Our main result shows that a desirable expansion policy is one, under which the induced transition probabilities are similar for each starting state.
More generally, we show that the gradient variance of \treepol{} decays at a rate of $|\lambda_2|^d,$ where $d$ is the depth of the tree and $\lambda_2$ is the second eigenvalue of the transition matrix induced by the tree expansion policy. This work is the first to prove such a relation between PG variance and tree expansion policy. In addition, we prove that with an approximate forward model, the bias of the gradient is bounded proportionally to the approximation error of the model.

To verify our results, we implemented a practical version of \treepol{} that exhaustively searches the entire tree and applies a neural network on its leaves. We test our algorithm on a parallelized Atari GPU simulator \citep{dalton2020accelerating}. To enable a tractable deep search, up to depth eight, we also introduce a pruning technique that limits the width of the tree. We do so by sampling only the most promising nodes at each level.  We integrate our \treepol\ GPU implementation into the popular PPO \citep{schulman2017proximal} and compare it to the flat distributed variant of PPO. This allows us to demonstrate the potential benefit of utilizing learned models while isolating the fundamental properties of TS without added noise. In all tested Atari games, our results outperform the baseline and obtain up to 5x more reward. We further show in Section~\ref{sec: experiments} that the associated gradient variance is smaller by three orders of magnitude in all games, demonstrating the relation between low gradient variance and high reward.

We summarize our key contributions. 
(i) We show how to combine two families of SoTA approaches: PG and TS by \textbf{introducing \treepol:} a %A
novel parametric policy that generalizes softmax to planning. Specifically, we propose two variants based on cumulative and exponentiated rewards.
(ii) We \textbf{prove that the gradient variance of \treepol\ in its two variants decays } with its tree depth.  Our analysis sheds new light on the choice of tree expansion policy. It raises the question of optimality in terms of variance versus the traditional regret; e.g., in UCT \citep{kocsis2006bandit}. 
(iii) We prove that with an approximate forward model, the \textbf{gradient bias is proportional to the approximation error}, while retaining the variance decay. This quantifies the accuracy required from a learned forward model. 
(iv) We \textbf{implement a differentiable deep version of \treepol} that employs a parallelized GPU tree expansion. We demonstrate how its gradient variance is reduced by three orders of magnitude over PPO while obtaining up to 5x reward.

\section{Preliminaries}
\label{sec:preliminaries}
Let $\Delta_U$ denote simplex over the set $U.$ 
Throughout, we consider
a discounted Markov Decision Process (MDP) $\mathcal{M} = (\cS, \cA, P, r,\gamma, \nu)$, where $\cS$ is a finite state space of size $S$, $\cA$ 
is a finite action space of size $A$, $r: \cS \times \cA \to [0,1]$ is the reward function, $P: \cS \times \cA \to \Delta_\cS$ is the transition function, $\gamma \in (0,1)$ is the discount factor, and $\nu \in \bR^S$ is the initial state distribution. 
We denote the transition matrix starting from state $s$ by $P_s \in [0,1]^{A\times S},$ i.e., $[P_s]_{a,s'} = P(s'|a,s).$ Similarly, let $R_s = r(s, \cdot) \in \bR^A$ denote the corresponding reward vector.
Separately, %Further 
let $\pi: \cS \rightarrow \Delta_\cA$ be a stationary policy. Let $P^{\pi}$ and $R_\pi$ be the induced transition matrix and reward function, respectively, i.e., $P^{\pi}(s'|s) = \sum_a \pi(a|s) \Pr(s'|s, a)$ and  $R_{\pi}(s) = \sum_a \pi(a|s) r(s, a)$.
Denote the stationary distribution of $P^\pi$ by $\mu_\pi \in \bR^{S}$  s.t. $\mu_\pi ^\top P^\pi = P^\pi,$ and the discounted state visitation frequency by $d_\pi$ so that $d_\pi^\top = (1 - \gamma) \sum_{t=0}^{\infty} \gamma^t \nu^\top (P^\pi)^t.$
Also, let $V^\pi \in \bR^S$ be the value function of $\pi$ defined by $V^\pi(s) = \bE^\pi \left[ \sum_{t=0}^\infty \gamma^t r\left(s_t,\pi(s_t)\right)  \mid s_0 = s \right]$, and let $Q^\pi \in \bR^{S \times A}$ be the Q-function such that $Q^\pi(s,a) = \bE^\pi\left[r(s, a) +\gamma V^\pi(s')\right]$. Our goal is to find an optimal policy $\pi^\star$ such that
$
    V^\star(s)
    \equiv
    V^{\pi^\star}(s)
    =
    \max_{\pi} V^\pi(s),~ \forall s \in \cS.
$

For the analysis in Section~\ref{sec:theory}, we introduce the following
%vector 
notation. Denote by $\Theta \in \bR^S$ the vector representation of $\theta(s)~\forall s \in \cS.$ For a vector $u$, denote by $\exp(u)$ the coordinate-wise exponent of $u$ and by $D(u)$ the diagonal square matrix with $u$ in its diagonal. For a matrix $A$, denote its $i$-th eigenvalue by $\lambda_i(A).$ Denote the $k$-dimensional identity matrix and all-ones vector by $I_k$ and $\ones_k,$ respectively. Also, denote the trace operator by $\Tr.$ Finally, we treat all vectors as column vectors. 
\subsection{Policy Gradient}
PG schemes seek to maximize the cumulative reward as a function of the policy $\pi_\theta(a|s)$ by performing gradient steps on $\theta$. 
The celebrated Policy Gradient Theorem \citep{sutton1999policy} states that
\begin{equation*}
    \frac{\partial}{\partial \theta}  \nu^\top V^{\pi_\theta} = \bE_{s\sim{d_{\pi_\theta}} ,a\sim\pi_\theta(\cdot|s)}\left[\nabla_\theta\log\pi_\theta(a|s)Q^{\pi_\theta}(s,a)\right],
\end{equation*}
where $\nu$ and $d_{\pi_\theta}^\tr$ are as defined above. % \gug{I am not sure if the below expression is the variance of the 
The variance of the gradient is thus
\begin{equation}\label{eq:grad_var}
\Var_{s\sim{d_{\pi_\theta}},a\sim\pi_\theta(\cdot|s)}\left(\nabla_\theta\log\pi_\theta(a|s)Q^{\pi_\theta}(s,a)\right).
\end{equation}

In the notation above, we denote the variance of a vector random variable $X$ by
\begin{equation*}
\Var_{x }\left(X\right) = \Tr  \left[ \bE_{x }\left[\left(X - \bE_{x }X\right)^\top \left(X- \bE_{x }X \right) \right] \right],
\end{equation*}
 similarly as in \citep{greensmith2004variance}. From now on, we drop the subscript from $\Var$ in \eqref{eq:grad_var} for brevity. When the action space is discrete, a commonly used parameterized policy is softmax: $\pi_\theta(a|s) \propto \exp \left(\theta(s, a) \right),$ where $\theta: \cS \times \cA \rightarrow \bR$ is a state-action parameterization.

\section{\treepol: Exponent of trajectories}
\label{sec: softtreemax}
We introduce a new family of policies called \treepol, which are a model-based generalization of the popular softmax. We propose two variants: Cumulative (\treepolC) and Exponentiated (\treepolE). In both variants, we replace the generic softmax logits $\theta(s,a)$ with the score of a trajectory of horizon $d$ starting from $(s,a),$
generated by applying a behavior policy $\pib$. In \treepolC, we exponentiate the expectation of the logits. In \treepolE, we first exponentiate the logits and then only compute their expectation.

\textbf{Logits}. We define the \treepol\ logit $\ell_{s,a}(d;\theta)$ to be the random variable depicting the score of a trajectory of horizon $d$ starting from $(s,a)$ and following the policy $\pib$:
\begin{equation}\label{eq:logit}
     \ell_{s, a}(d;\theta) = \gamma^{-d} \left[\sum_{t=0}^{d-1} \gamma^t r_t + \gamma^d \theta(s_d)\right].
\end{equation} 
In the above expression, note that  
%Namely, 
$s_0 = s,~a_0 = a,~a_t \sim \pib(\cdot|s_t)~\forall t \geq 1,$ and $r_t\equiv r\left(s_t,a_t\right).$ 
For brevity of the analysis, we let the parametric score $\theta$ in \eqref{eq:logit} be state-based, similarly to a value function. Instead, one could use a state-action input analogous to a Q-function. Thus, \treepol{} can be integrated into the two types of implementation of RL algorithms in standard packages. Lastly, the preceding $\gamma^{-d}$ scales the $\theta$ parametrization to correspond to its softmax counerpart.

\textbf{\treepolC}. Given an inverse temperature parameter $\beta$, 
 we let \treepolC\ be
    \begin{equation}\label{eq:polC}
        \pid^{\txC} (a|s)  \propto \exp \left[\beta \bE^{\pib} \ell_{s, a}(d;\theta) \right].      
    \end{equation}   
\treepolC\ gives higher weight to actions that result in higher expected returns. While standard softmax relies entirely on parametrization $\theta,$ \treepolC\ also interpolates a Monte-Carlo portion of the reward. 

\textbf{\treepolE}. The second operator we propose is \treepolE: 
\begin{equation}\label{eq:polE}
        \pid^{\txE} (a|s)  \propto \bE^{\pib} \exp \left[ \left( \beta  \ell_{s, a}(d;\theta)   \right) \right];
\end{equation} 
here, the expectation is taken outside the exponent. This objective corresponds to the exponentiated reward objective which is often used for risk-sensitive RL \citep{howard1972risk, fei2021exponential, noorani2021risk}. The common risk-sensitive objective is of the form $\log \bE[\exp(\delta R)],$ where $\delta$ is the risk parameter and $R$ is the cumulative reward. Similarly to that literature, the exponent in \eqref{eq:polE} emphasizes the most promising trajectories. 
% For $\beta \rightarrow \infty$, the trajectory with the highest value determines the chosen action.

\textbf{\treepol\ properties}. \treepol\ is a natural model-based generalization of softmax. For $d=0$, both variants above coincide since \eqref{eq:logit} becomes deterministic. In that case, for a state-action parametrization, they reduce to standard softmax.
When $\beta \rightarrow 0$, both variants again coincide and sample actions uniformly (exploration). When $\beta \rightarrow \infty,$ the policies become deterministic and greedily optimize for the best trajectory (exploitation). For \treepolC, the best trajectory is defined in expectation, while for \treepolE{} it is defined in terms of the best sample path.

\textbf{\treepol\ behavior policy selection}. The choice of behavior policy $\pi_b$ plays a crucial role in the performance of \treepol. As we show in Section~\ref{sec:theory}, the gradient variance is minimized when the transitions induced by $\pi_b$ result in similar distributions across states, which is achieved when the second eigenvalue of $P^{\pi_b}$ is small. Without specific assumptions on the MDP, a uniform policy that smoothens transition probabilities across all actions serves as a reasonable approximation to this goal. In practice, this leads to better mixing properties in the associated Markov chain.

\textbf{\treepol\ convergence.} Under regularity conditions, for any parametric policy, PG converges to local optima \citep{bhatnagar2009natural}, and thus also \treepol. For softmax PG, asymptotic \citep{agarwal2021theory} and rate results \citep{mei2020global} were recently obtained, by showing that the gradient is strictly positive everywhere \citep[Lemmas~8-9]{mei2020global}. We conjecture that \treepol\ satisfies the same property, being a generalization of softmax, but formally proving it is subject to future work.

\textbf{\treepol\ gradient.}
The two variants of \treepol\ involve an expectation taken over $S^d$ many trajectories from the root state $s$ and weighted according to their probability. 
Thus, during the PG training process, the gradient $\nabla_\theta \log \pi_\theta$ is calculated using a weighted sum of gradients over all reachable states starting from $s$.  
Our method exploits the exponential number of trajectories to reduce the variance while improving performance.
Indeed,
in the next section we prove that the gradient variance of \treepol{} decays exponentially fast as a function of the behavior policy $\pib$ and trajectory length $d$. In the experiments in Section~\ref{sec: experiments}, we also show how the practical version of \treepol{} achieves a significant reduction in the noise of the PG process and leads to faster convergence and higher reward.

\section{Theoretical Analysis}
\label{sec:theory}
In this section, we first bound the variance of PG when using the \treepol{} policy. Later, we discuss how the gradient bias resulting due to approximate forward models diminishes as a function of the approximation error, while retaining the same variance decay.

We show that the variance decreases  with the tree depth, and the rate is determined by the second eigenvalue of the transition kernel induced by $\pib.$ Specifically, we bound the same expression for variance as appears in \citep[Sec.~3.5]{greensmith2004variance} and  \citep[Sec.~A, Eq.~(21)]{wu2018variance}.
 Other types of analysis could instead have focused on the estimation aspect in the context of sampling \citep{zhang2021convergence, shen2019hessian, pham2020hybrid}. Indeed, in our implementation in Section~\ref{sec: deep rl implementation}, we manage to avoid sampling and directly compute the expectations in Eqs. \eqref{eq:polC} and \eqref{eq:polE}. As we show later, we do so by leveraging efficient parallel simulation on the GPU in feasible run-time. In our application, due to the nature of the finite action space and quasi-deterministic Atari dynamics \citep{bellemare2013arcade}, our expectation estimator is noiseless. We encourage future work to account for the finite-sample variance component. We defer all the proofs to Appendix~\ref{app:proofs}.

We begin with a general variance bound that holds for {\em any} parametric policy.

\begin{lemma}[Bound on the policy gradient variance]\label{lem:var_bound}
    Let $\nabla_\theta\log \pi_\theta(\cdot|s) \in \mathbb{R}^{A  \times \dim(\theta)}$ be a matrix whose $a$-th row is $\nabla_\theta\log \pi_\theta(a|s)^\top$. For any parametric policy $\pi_\theta$ and function $Q^{\pi_\theta}:\cS\times\cA \rightarrow \bR,$ 
\begin{align*} 
\Var\left(\nabla_\theta\log \pi_\theta(a|s) Q^{\pi_\theta}(s,a)\right) 
\leq & \\ \max_{s,a} \left[Q^{\pi_\theta}(s,a)\right]^2   \max_s & \| \nabla_\theta\log \pi_\theta(\cdot|s)\|_F^2.
\end{align*}
\end{lemma}

Hence, to bound \eqref{eq:grad_var}, it is sufficient to bound the Frobenius norm $\|\nabla_\theta\log \pi_\theta(\cdot|s)\|_F$ for any $s$.

Note that \treepol\ does not reduce the gradient uniformly, which would have been equivalent to a trivial change in the learning rate. While the gradient norm shrinks, the gradient itself scales differently along the different coordinates. This scaling occurs along different eigenvectors, as a function of problem parameters ($P$, $\theta$) and our choice of behavior policy ($\pi_b$), as can be seen in the proof of the upcoming Theorem~\ref{thm:rate_result}. This allows \treepol\ to learn a good ``shrinkage'' that, while reducing the overall gradient, still updates the policy quickly enough. This reduction in norm and variance resembles the idea of gradient clipping \cite{zhang2019gradient}, where the gradient is scaled to reduce its variance, thus increasing stability and improving overall performance.

A common assumption in the RL literature \citep{szepesvari2010algorithms} that we adopt for the remainder of the section is that the transition matrix $P^{\pib},$ induced by the behavior policy $\pib,$ is irreducible and aperiodic. 
Consequently, its second highest eigenvalue satisfies $|\lambda_2(P^{\pib})| < 1.$

From now on, we divide the variance results for the two variants of \treepol\ into two subsections. For \treepolC, the analysis is simpler and we provide an exact bound. The case of \treepolE\ is more involved and we provide for it a more general result. In both cases, we show that the variance decays exponentially with the planning horizon. 
\subsection{Variance of \treepolC}

We express \treepolC\ in vector form as follows.

\begin{lemma}[Vector form of \treepolC]\label{def:pidC} For $d\geq 1,$ \eqref{eq:polC} is given by
\begin{equation}\label{eq:SoftTreeMax_matrix} 
    \pid^{\txC}(\cdot|s) = \frac{ \exp\left[\beta \left(\Csd +  P_s \left(P^{\pib}\right)^{d-1} \Theta \right)\right]}{\ones_A^\top \exp\left[\beta \left(\Csd +  P_s \left(P^{\pib}\right)^{d-1} \Theta \right)\right]},
    \end{equation}
        where
    \begin{equation*}
        \Csd = \gamma^{-d} R_s + P_s \left[\sum_{h=1}^{d-1} \gamma^{h - d} \left(P^{\pib}\right)^{h - 1} \right] R_{\pib}. 
    \end{equation*}
\end{lemma}

The vector $\Csd \in \bR^{A}$ represents the cumulative discounted reward in expectation along the trajectory of horizon $d.$ This trajectory starts at state $s,$ involves an initial reward dictated by $R_s$ and an initial transition as per $P_s.$ Thereafter, it involves rewards and transitions specified by $R_{\pib}$ and $P^{\pib},$ respectively. Once the trajectory reaches depth $d,$ the score function $\theta(s_d)$ is applied,.

\begin{lemma}[Gradient of \treepolC]\label{lem:analytic_gradC}
The \treepolC\ gradient is given by
\begin{align*} 
 \nabla_\theta\log \pid^{\txC} =\beta\left[I_{A} -  \ones_A {(\pid^{\txC}})^\top \right]P_s \left(P^{\pib}\right)^{d-1},
\end{align*}
in $\bR^{A \times S},$ where for brevity, we drop the $s$ index in the policy above, i.e., ${\pid^{\txC} \equiv \pid^{\txC}(\cdot|s).}$
\end{lemma}
We are now ready to present our first main result:
\begin{theorem}[Variance decay of \treepolC]\label{thm:rate_result}
For every $Q: \cS \times \cA \rightarrow \bR,$ the \treepolC\ policy gradient variance is bounded by
\begin{align*}
\Var \left(\nabla_\theta \log \pid^{\txC}(a|s) Q(s,a)\right)  
\leq& \\ 2\frac{A^2 S^2\beta^2}{(1 - \gamma)^2} &|\lambda_2(P^{\pib})|^{2(d-1)}.
\end{align*}
\end{theorem}

We provide the full proof in Appendix~\ref{app:thm1proof}, and  briefly outline its essence here.
% \vspace{-0.35cm}
\begin{proof}[Proof outline]
Lemma~\ref{lem:var_bound} allows us to bound the variance using a direct bound on the gradient norm. 
The gradient is given in Lemma \ref{lem:analytic_gradC} as a product of three matrices, which we now study from right to left. The matrix $P^{\pib}$ is a row-stochastic matrix. Because the associated Markov chain is irreducible and aperiodic, it has a unique stationary distribution. This implies that $P^{\pib}$ has one and only one eigenvalue equal to $1;$ all others have magnitude strictly less than $1.$ Let us suppose that all these other eigenvalues have multiplicity $1$ (the general case with repeated eigenvalues can be handled via Jordan decompositions as in \citep[Lemma1]{pelletier1998almost}). Then, $P^{\pib}$ has the spectral decomposition
$
    P^{\pib} = \ones_S \mu^\top_{\pib} + \sum_{i = 2}^{S} \lambda_i v_i u_i^\top,
$
where $\lambda_i$ is the $i$-th eigenvalue of $P^{\pib}$ (ordered in descending order according to their magnitude) and $u_i$ and $v_i$ are the corresponding left and right eigenvectors, respectively, and therefore $(P^{\pib})^{d-1} = \ones_S \mu^\top_{\pib} + \sum_{i = 2}^{S} \lambda_i^{d-1} v_i u_i^\top.$
The second matrix in the gradient relation in Lemma~\ref{lem:analytic_gradC}, $P_s,$ is a rectangular transition matrix that translates the vector of all ones from dimension $S$ to $A:$ $P_s \ones_S = \ones_A.$
Lastly, the first matrix %in Lemma~\ref{lem:analytic_gradC}
$\left[I_{A} -  \ones_A {(\pid^{\txC}})^\top \right]$ 
is a projection whose null-space includes
the vector $\ones_A,$ i.e., $\left[I_{A} -  \ones_A {(\pid^{\txC}})^\top \right] \ones_A = 0.$ 
Combining the three properties above when multiplying the three matrices of the gradient, it is easy to see that the first term in the expression for $(P^{\pib})^{d-1}$ gets canceled, and we are left with bounded summands scaled by $\lambda_i(P^{\pib})^{d-1}.$ Recalling that $|\lambda_i(P^{\pib})| <1$ and that $|\lambda_2| \geq |\lambda_3| \geq \dots$ for $i = 2, \dots, S,$ we obtain the desired result. 
\end{proof}

 It's important to note that \treepol\ does not reduce the gradient uniformly, which would be equivalent to simply decreasing the learning rate. Rather, the gradient is scaled differently along different eigenvectors, with scaling factors that depend on the MDP structure and behavior policy $\pi_b$. This non-uniform scaling allows the policy to continue learning effectively while reducing harmful variance. The empirical results in Section~\ref{sec: experiments} demonstrate that this variance reduction leads to faster, more stable convergence rather than slowing it down.

Theorem~\ref{thm:rate_result} guarantees that the variance of the gradient decays  with $d.$ More importantly, it also provides a novel insight for choosing the behavior policy $\pib$ as the policy that minimizes the absolute second eigenvalue of the $P^{\pib}.$ Indeed, the second eigenvalue of a Markov chain relates to its connectivity and its rate of convergence to the stationary distribution \citep{levin2017markov}.

\textbf{Optimal variance decay}. For the strongest reduction in variance, the behavior policy $\pib$ should be chosen to achieve an induced Markov chain whose transitions are state-independent. In that case, $P^{\pib}$ is a rank one matrix of the form $\ones_S \mu_{\pib}^\top,$ and $\lambda_2(P^{\pib}) = 0.$ Then, $\Var\left(\nabla_\theta \log \pi_\theta(a|s) Q(s,a)\right) = 0.$ Naturally, 
this can only be done for pathological MDPs; see Appendix~\ref{sec: zero grad} for a more detailed discussion. Nevertheless, as we show in Section~\ref{sec: deep rl implementation}, we choose our tree expansion policy to reduce the variance as best as possible.

\textbf{Worst-case variance decay}. In contrast, and somewhat surprisingly, when $\pib$ is chosen so that the dynamics is deterministic, there is no guarantee that it will decay exponentially fast. For example, if $P^{\pib}$ is a permutation matrix, then $\lambda_2(P^{\pib}) = 1,$ and advancing the tree amounts to only updating the gradient of one state for every action, as in the basic softmax.

\subsection{Variance of \treepolE}
The proof of the variance bound for \treepolE\ is similar to that of \treepolC, but more involved. It also requires the assumption that the reward depends only on the state, i.e. $r(s,a)\equiv r(s).$ This is indeed the case in most standard RL environments such as Atari and Mujoco. 

\begin{lemma}[Vector form of \treepolE]\label{def:pidE}
For $d\geq1$, \eqref{eq:polE} is given by
\begin{equation}\label{eq:SoftTreeMax_matrix2}
    \pid^{\txE}(\cdot|s)= \frac{\Esd \exp(\beta \Theta)}{1_A^\top  \Esd \exp(\beta  \Theta)}, 
    \end{equation}
    where
    \begin{equation*}
    \Esd = P_s  \prod_{h=1}^{d-1} \left(D\left(\exp(\beta \gamma^{h-d} R)\right) P^{\pib} \right).
\end{equation*}
The vector $R$ above is the $S$-dimensional vector whose $s$-th coordinate is $r(s).$
\end{lemma}
The matrix $\Esd \in \bR^{A \times S}$ has a similar role to $\Csd$ from \eqref{eq:SoftTreeMax_matrix}, but it represents the exponentiated cumulative discounted reward. Accordingly, it is a product of $d$ matrices as opposed to a sum.  It captures the expected reward sequence starting from $s$ and then iteratively following $P^{\pib}.$ After $d$ steps, we apply the score function on the last state as in \eqref{eq:SoftTreeMax_matrix2}.

\begin{lemma}[Gradient of \treepolE]\label{lem:analytic_gradE}
The \treepolE\ gradient is given by
\begin{align*}
& \nabla_\theta\log \pid^{\txE} = \beta\left[I_{A} - \ones_A (\pid^{\txE})^\top \right] \times \\ 
& \frac{ D\left(\pid^{\txE}\right)^{-1} \Esd D(\exp(\beta \Theta))}{{\bf{1}}^{\top}_A \Esd \exp(\beta \Theta)}\quad \in \quad \bR^{A \times S},
\end{align*}
where for brevity, we drop the $s$ index in the policy above, i.e., ${\pid^{\txE} \equiv \pid^{\txE}(\cdot|s).}$
\end{lemma}

This gradient structure is harder to handle than that of \treepolC\ in Lemma~\ref{lem:analytic_gradC}, but here we also can bound the decay of the variance nonetheless.
\begin{theorem}[Variance decay of \treepolE]\label{thm:rate_result2}
There exists $\alpha \in \left(0,1\right)$ such that,
\begin{equation*}
\Var\left(\nabla_\theta \log \pid^{\txE}(a|s) Q(s,a)\right) \in \mathcal{O}\left(\beta^2  \alpha^{2d}\right),
\end{equation*}
for every $Q.$ Further, if $P^{\pib}$ is reversible or if the reward is constant,
then $\alpha = |\lambda_2(P^{\pib})|$.
\end{theorem}

\textbf{Theory versus Practice.} We demonstrate the above result in simulation. We draw a random finite MDP, parameter vector $\Theta \in \bR^S_+,$ and behavior policy $\pib.$ We then empirically compute the PG variance of \treepolE\ as given in \eqref{eq:grad_var} and compare it to $|\lambda_2(P^{\pib})|^d.$ We repeat this experiment three times for different $P^{\pib}:$ (i) close to uniform, (ii) drawn randomly, and (iii) close to a permutation matrix. As seen in Figure \ref{fig: toy simulation}, the empirical variance and our bound match almost identically. 
This also suggests that $\alpha = |\lambda_2(P^{\pib})|$  in the general case and not only when $P^{\pib}$ is reversible or when the reward is constant.

\begin{figure}
  \begin{minipage}[c]{0.5\textwidth}
    \includegraphics[width=1.1\columnwidth]{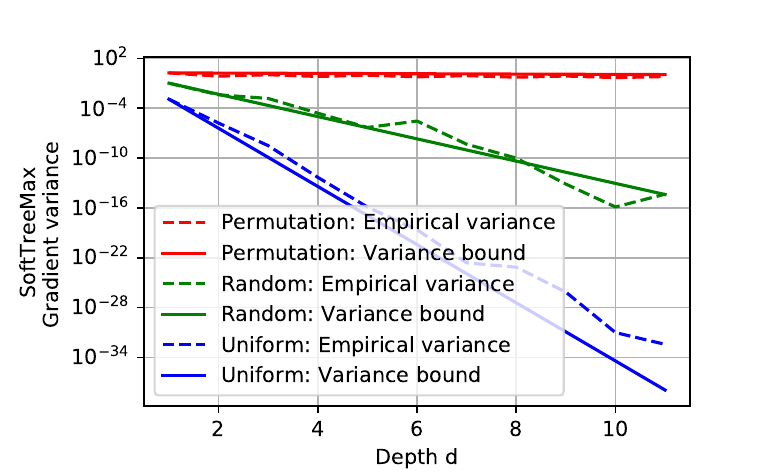}
  \end{minipage}\hfill
  \begin{minipage}[c]{0.45\textwidth}
    \caption{A comparison of the empirical PG variance and our bound for \treepolE\ on randomly drawn MDPs. We present three cases for $P^{\pib}:$ (i) close to uniform, (ii) drawn randomly, and (iii) close to a permutation matrix. This experiment verifies the optimal and worse-case rate decay cases. The variance bounds here are taken from Theorem~\ref{thm:rate_result2} where we substitute $\alpha=|\lambda_2(P^{\pib})|.$ To account for the constants, we match the values for the first point in $d=1.$}
    \label{fig: toy simulation}
  \end{minipage}
\end{figure}

\subsection{Bias with an Approximate Forward Model}

The definition of the two \treepol{} variants involves the knowledge of the underlying environment, in particular the value of $P$ and $r.$ However, in practice, we often can only learn approximations of the dynamics from interactions, e.g., using NNs \citep{ha2018world, schrittwieser2020mastering}. Let $\hat{P}$ and $\h{r}$ denote the approximate kernel and reward functions, respectively. In this section, we study the consequences of the approximation error on the \treepolC{}  gradient. 

Let $\pidhat^{\txC}$ be the \treepolC{} policy defined given the approximate forward model introduced above. That is, let $\pidhat^{\txC}$ be defined exactly as in  \eqref{eq:SoftTreeMax_matrix}, but using $\h{R}_s, \h{P}_s, \h{R}_{\pi_b}$ and $\h{P}^{\pi_b},$ instead of their unperturbed counterparts from Section~\ref{sec:preliminaries}. Then, the variance of the corresponding gradient again decays exponentially with a decay rate of $\lambda_2(\h{P}^{\pib}).$  However, a gradient bias is introduced. In the following, we bound this bias in terms of the approximation error and other problem parameters. The proof is provided in Appendix \ref{app:biasest}.

\begin{theorem}\label{thm:bias}
Let $\epsilon$ be the maximal model mis-specification, i.e., let 
$\max\{\|P - \hP\|, \|r- \hr\|\} = \epsilon.$ 
Then the policy gradient bias due to $\pidhat^{\txC}$ satisfies  
\begin{align} \label{eq: bias bound}
    \left\|\frac{\partial}{\partial \theta}  \left(\nu^\top V^{\pid^{\txC}}\right) -  \frac{\partial}{\partial \theta}  \left(\nu^\top V^{\pidhat^{\txC}}\right) \right\| = \quad \quad \quad & \\ 
    \mathcal{O}\left(\frac{1}{(1-\gamma)^2 \gamma^d} S \beta^2 d  \epsilon \right). & \nonumber
\end{align}
\end{theorem}

% \begin{theorem}\label{thm:bias}
% Let $\epsilon$ be the maximal model mis-specification, i.e., let 
% $\max\{\|P - \hP\|, \|r- \hr\|\} = \epsilon.$ 
% Then the policy gradient bias due to $\pidhat^{\txC}$ satisfies  
% %
% \begin{align} \label{eq: bias bound}
%     \left\|\frac{\partial}{\partial \theta}  \left(\nu^\top V^{\pid^{\txC}}\right) -  \frac{\partial}{\partial \theta}  \left(\nu^\top V^{\pidhat^{\txC}}\right) \right\| = \quad \quad \quad & \\ 
%     \mathcal{O}\left(\frac{1}{(1-\gamma)^2} S \beta^2 d \epsilon \right). & \nonumber
% \end{align}
% \end{theorem}

\begin{proof}[Proof outline]

First, we prove that $\max\{\|R_s - \h{R}_s\|,\|P_s - \h{P}_s\|,\|R_{\pi_b} - \h{R}_{\pi_b}\|,\|P^{\pi_b} - \h{P}^{\pi_b}\|\} = \mathcal{O}(\epsilon).$ This follows from the fact that the differences above are suitable convex combinations of either the rows of $P-\h{P}$ or $r-\h{r}.$ 
We use the above observation along with the definitions of $\pid^{\txC}$ and $\pidhat^{\txC}$ given in \eqref{eq:SoftTreeMax_matrix} to show that $\|\pid^{\txC} - \pidhat^{\txC}\| = O(\beta d \epsilon).$ The proof for the latter builds upon two key facts: (a) $\|(P^{\pib})^k - (\hat{P}^{\pib})^k\| \leq \sum_{h = 1}^k \|\h{P}^{\pib}\|^{h - 1} \|\h{P}^{\pib} - P^{\pib} \| \|p^{\pib}\|^{k - h} =   O(k \epsilon)$ for any $k \geq 0$, and (b)  $|e^x - 1| = O(x)$ as $x \to 0.$ Next, we decompose the LHS of \eqref{eq: bias bound} to get
\begin{align*}
 &\sum_{s} \left( \prod_{i = 1}^4 X_i(s) - \prod_{i = 1}^4 \h{X}_i(s) \right) 
       = \\
  &     \sum_s \sum_{i = 1}^4 \h{X}_1(s) \cdots \h{X}_{i - 1}(s) \left(X_i(s) - \h{X}_i(s)\right) \times \\
   & \quad \quad \quad \quad \quad\quad \quad \quad  X_{i + 1}(s) \cdots X_4(s),
\end{align*}
where $X_1(s) = d_{\pid^{\txC}}(s) \in \bR,$ $X_2(s) = (\nabla_\theta\log\pid^{\txC}(\cdot|s))^\tr \in \bR^{S \times A},$ $X_3(s) =  D(\pid^{\txC}(\cdot|s)) \in \bR^{A \times A},$ $X_4(s) = Q^{\pid^{\txC}}(s, \cdot) \in \bR^{A \times A},$ and $\h{X}_1(s), \ldots, \h{X}_4(s)$ are similarly defined with $\pid^{\txC}$ replaced by $\pidhat^{\txC}.$ Then, we show that, for $i = 1, \ldots, 4,$ (i) $\|X_i(s) - \h{X}_i(s)\| = O(\epsilon)$ and (ii) $\max\{\|X_i\|,\|\h{X}_i\|\}$ is bounded by problem parameters. From this, the desired result follows.
\end{proof}

To the best of our knowledge, Theorem \ref{thm:bias} is the first result that bounds the bias of the gradient of a parametric policy due to an approximate model. This theorem reveals an intriguing trade-off in \treepol: while the variance decays exponentially with tree depth $d$ as $O(|\lambda_2(P^{\pi_b})|^{d-1})$ according to Theorem \ref{thm:rate_result}, the gradient bias with an approximate model grows as $O(d\gamma^{-d})$. Since $\gamma^{-d}$ grows exponentially with $d$, there exists an optimal tree depth that balances these opposing effects.

The bias also depends on the temperature parameter $\beta$, with higher temperature (lower $\beta$) reducing bias at the expense of more exploratory policies. In the extreme case of $\beta=0$, the policy becomes uniform with no bias. Additionally, the error scales linearly with $d$ because the policy suffers from cumulative error as it relies on further-looking states in the approximate model.

These results suggest that if the learned model is accurate enough, we can expect similar convergence properties for \treepolC{} as we would obtain with the true dynamics. In practice, one should adjust both $d$ and $\beta$ based on the estimated accuracy of the forward model to achieve the best performance. This is particularly important for practitioners implementing \treepol{} with learned forward models.

\section{\treepol: Deep Parallel Implementation}
\label{sec: deep rl implementation}

\begin{figure*}[!ht]
\centering
\includegraphics[width=1.2\columnwidth]{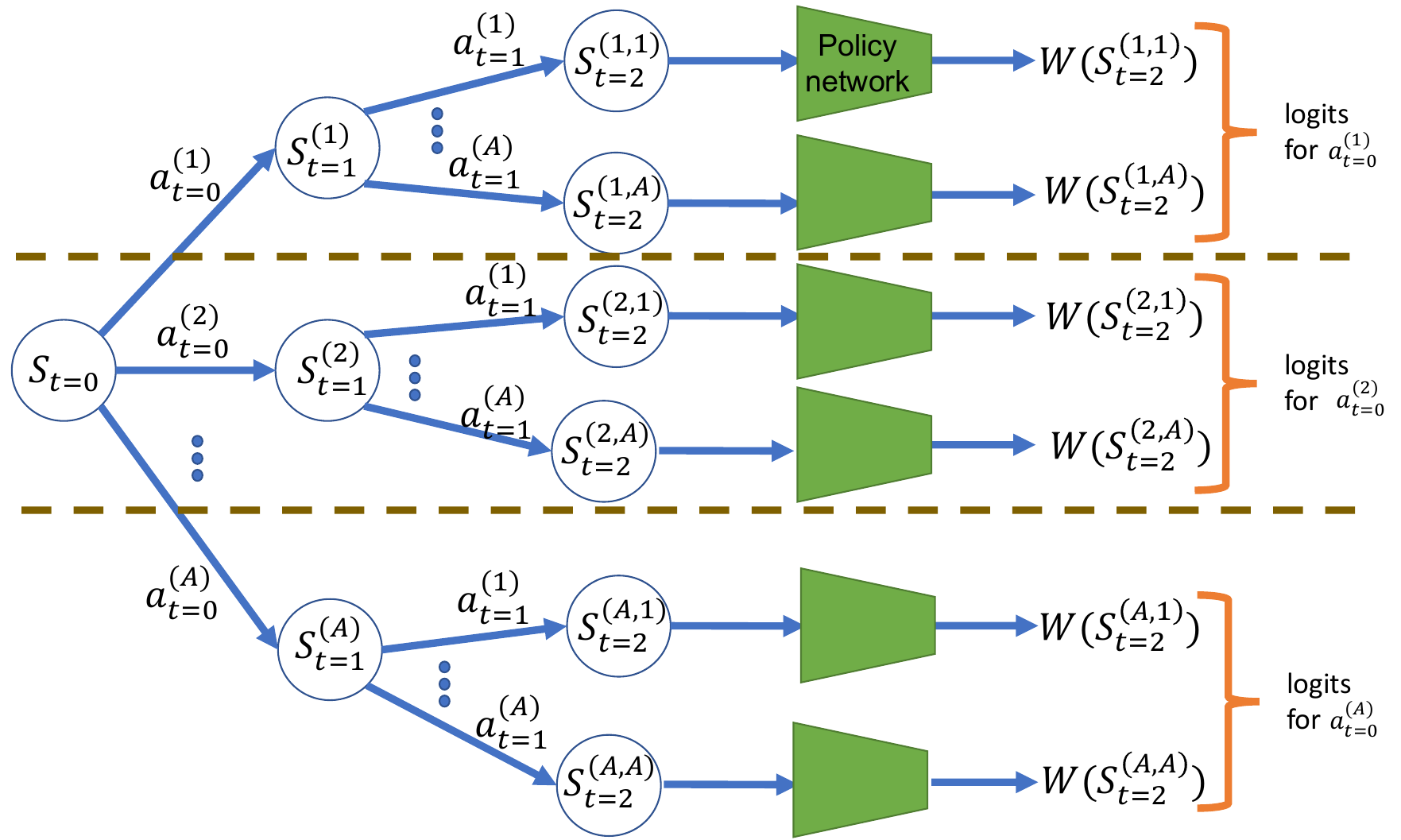}
  \caption{\textbf{\treepol{} policy}. Our exhaustive parallel tree expansion iterates on all actions at each state up to depth $d$ ($=2$ here). The leaf state of every trajectory is used as input to the policy network.  The output is then added to the trajectory's cumulative reward as described in \eqref{eq:logit}. I.e., instead of the standard softmax logits, we add the cumulative discounted reward to the policy network output. This policy is differentiable and can be easily integrated into any PG algorithm. In this work, we build on PPO and use its loss function to train the policy network.  }  
  \label{fig:policy_diagram}
\end{figure*}

Following impressive successes of deep RL \citep{mnih2015human, silver2016mastering}, using deep NNs in RL is standard practice. Depending on the RL algorithm, a loss function is defined and gradients on the network weights can be calculated.
In PG methods, the scoring function used in the softmax is commonly replaced by a neural network $W_\theta$: $\pi_\theta(a|s) \propto \exp\left(W_\theta(s,a)\right).$ Similarly, we implement \treepol{} by replacing $\theta(s)$ in \eqref{eq:logit} with a neural network $W_\theta(s)$. Although both variants of \treepol\ from Section~\ref{sec: softtreemax} involve computing an expectation, this can be hard in general. One approach to handle it is with sampling, though these introduce estimation variance into the process. We leave the question of sample-based theory and algorithmic implementations for future work.

Instead, in finite action space environments such as Atari, we compute the exact expectation in \treepol\ with an exhaustive TS of depth $d$. Despite the exponential computational cost of spanning the entire tree, recent advancements in parallel GPU-based simulation allow efficient expansion of all nodes at the same depth simultaneously \citep{dalal2021improve,rosenberg2022planning}. This is possible when a simulator is implemented on GPU \citep{dalton2020accelerating, makoviychuk2021isaac, freeman2021brax}, or when a forward model is learned \citep{kim2020learning,ha2018world}. To reduce the complexity to be linear in depth, we apply tree pruning to a limited width in all levels. We do so by sub-sampling only the most promising branches at each level. Limiting the width drastically improves runtime, and enables respecting GPU memory limits, with only a small sacrifice in performance.

To summarize, in the practical $\treepol{}$ algorithm we perform an exhaustive tree expansion with pruning to obtain trajectories up to depth $d.$ We expand the tree with equal weight to all actions, which corresponds to a uniform tree expansion policy $\pib.$ We apply a neural network on the leaf states, and accumulate the result with the rewards along each trajectory to obtain the logits in \eqref{eq:logit}. Finally, we aggregate the results using \treepolC. We leave experiments \treepolE\ for future work on risk-averse RL. 

For a detailed illustration of our GPU-based tree expansion implementation, see Appendix~\ref{app:tree_expansion}. In addition, we provide the psudeocode for C-SoftTreeMax in Algorithm~\ref{alg:SoftTreeMax} and its integration with PPO in Algorithm~\ref{alg:SoftTreeMaxPPO} in Appendix~\ref{app:algorithms}. During training, the gradient propagates to the NN weights of $W_\theta.$ When the gradient $\nabla_\theta \log \pid$ is calculated at each time step, it updates $W_\theta$ for all leaf states, similarly to Siamese networks \citep{bertinetto2016fully}. An illustration of the policy is given in Figure \ref{fig:policy_diagram}.

\section{Experiments}
\label{sec: experiments}
We conduct our experiments on multiple games from the Atari simulation suite \citep{bellemare2013arcade}. As a baseline, we train a PPO \citep{schulman2017proximal} agent with $256$ GPU workers in parallel \citep{dalton2020accelerating}. For the tree expansion, we employ a GPU breadth-first as in \citep{dalal2021improve}. We then train \treepolC{} \footnote{We also experimented with \treepolE\ and the results were almost identical. This is due to the quasi-deterministic nature of Atari, which causes the trajectory logits \eqref{eq:logit} to have almost no variability. We encourage future work on E-SoftTreeMax using probabilistic environments that are risk-sensitive.} for depths $d=1 \dots 8,$ with a single worker. For depths $d \geq 3$, we limited the tree to a maximum width of $1024$ nodes and pruned trajectories with low estimated weights. Since the distributed PPO baseline advances significantly faster in terms of environment steps, for a fair comparison, we ran all experiments for one week on the same machine. For more details see Appendix \ref{app:experiments}.

In Figure~\ref{fig:variance_curves}, we plot the reward and variance of \treepol\ for each game, as a function of depth. The dashed lines are the results for PPO. Each value is taken after convergence, i.e., the average over the last $20\%$ of the run. The numbers represent the average over five seeds per game. 
% We choose to exclude the standard deviation to avoid excessive clutter in the plot. 
The plot conveys three intriguing conclusions. First, in all games, \treepol\ achieves significantly higher reward than PPO. Its gradient variance is also orders of magnitude lower than that of PPO. Second, the reward and variance are negatively correlated and mirror each other in almost all games. This phenomenon demonstrates the necessity of reducing the variance of PG for improving performance. 
Lastly, each game has a different sweet spot in terms of optimal tree depth. Recall that we limit the run-time in all experiments to one week 
The deeper the tree, the slower each step and the run consists of less steps. This explains the non-monotone behavior as a function of depth. 
For a more thorough discussion on the sweet spot of different games, see Appendix~\ref{sec: step based plots}.

\begin{figure*}[!htb]
\includegraphics[scale=0.48]{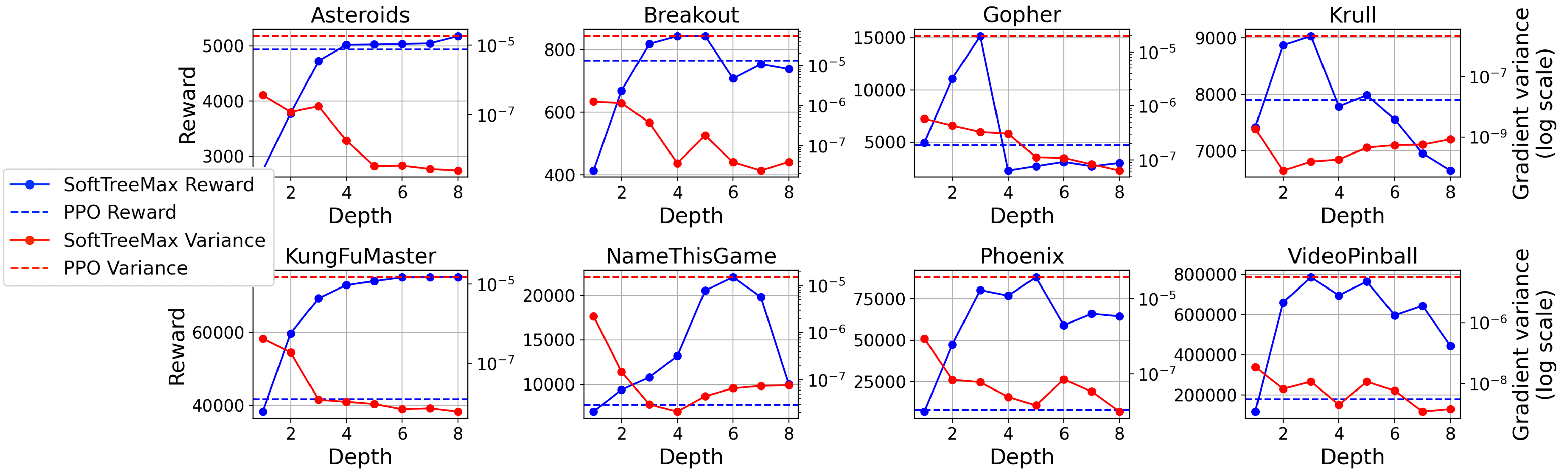}
   \caption{\textbf{Reward and Gradient variance: GPU \treepol{} 
   (single worker) vs PPO ($\bf{256}$ GPU workers).} The blue reward plots show the average  of $50$ evaluation episodes. The red variance plots show the average gradient variance of the corresponding training runs, averaged over five seeds. The dashed lines represent the same for PPO. Note that the variance y-axis is in log-scale.
   }  \label{fig:variance_curves}
\end{figure*}

\section{Related Work}

Reducing variance in PG estimates is essential for improving efficiency and stability. Approaches include baseline subtraction \citep{sutton1999policy, weaver2001optimal}, action-dependent baselines \citep{wu2018variance}, sub-sampling techniques like SVRPG \citep{papini2018stochastic}, and natural policy gradient \citep{kakade2001natural}. Multi-step returns for value estimation, such as n-step TD \citep{sutton1998reinforcement} and GAE \citep{schulman2015high}, primarily affect learning targets rather than policy parameterization. Our approach differs by integrating planning directly into the policy structure, making it fundamentally different from both variance reduction techniques and multi-step return methods.

\textbf{Softmax Operator.}
The softmax policy became a canonical part of PG to the point where theoretical results of PG focus specifically on it \citep{zhang2021convergence, mei2020global,li2021softmax, ding2022global}. Even though we focus on a tree extension to the softmax policy, our methodology is general and can be easily applied to other discrete or continuous parameterized policies as in \citep{mei2020escaping, miahi2021resmax, silva2019optimization}. 
It's important to distinguish between tree search and our approach of tree expansion. Traditional tree search methods like MCTS identify the best trajectory through a selection process, while our approach explores all possible trajectories up to a certain depth to compute an improved policy. Recent work by \citet{efroni2018beyond} showed that multi-step greedy policies improve convergence rates, and \citet{protopapas2024policy} combined policy mirror descent with lookahead planning, though using different mechanisms than our SoftTreeMax policy.

\textbf{Parallel Environments.}
 In this work we used accurate parallel models that are becoming more common with the increasing popularity of GPU-based simulation \citep{makoviychuk2021isaac, dalton2020accelerating, freeman2021brax}. Alternatively, in relation to Theorem~\ref{thm:bias}, one can rely on recent works that learn the underlying model \citep{ha2018world, schrittwieser2020mastering} and use an approximation of the true dynamics. 
\textbf{Risk Aversion.}
Previous work considered exponential utility functions for risk aversion \citep{chen2007risk, garcia2015comprehensive,fei2021exponential}. This utility function is the same as \treepolE{} formulation from \eqref{eq:polE}, but we have it directly in the policy instead of the objective.  
\textbf{Reward-free RL.}
We showed that the gradient variance is minimized when the transitions induced by the behavior policy $\pib$ are uniform. This is expressed by the second eigenvalue of the transition matrix $P^{\pib}$. This notion of uniform exploration is common to the reward-free RL setup \citep{jin2020reward}. Several such works also considered the second eigenvalue in their analysis \citep{liu2018simple, tarbouriech2019active}.

\section{Discussion and Future Work}

In this work, we introduced for the first time a differentiable parametric policy that combines TS with PG. We proved that \treepol{} is essentially a variance reduction technique and explained how to choose the expansion policy to minimize the gradient variance. It is an open question whether optimal variance reduction corresponds to the appealing regret properties the were put forward by UCT \citep{kocsis2006bandit}. We believe that this can be answered by analyzing the convergence rate of \treepol, relying on the bias and variance results we obtained here.

As the learning process continues, the norm of the gradient and the variance {\em both} become smaller. On the face of it, one can ask if the gradient becomes small as fast as the variance or even faster can there be any meaningful learning? As we showed in the experiments, learning happens because the variance reduces fast enough (a variance of 0 represents deterministic learning, which is fastest).

Our work can be extended to infinite action spaces, where the theoretical analysis would involve transition kernels rather than matrices while preserving the same non-expansive operator properties. For implementation, the tree of continuous actions can be expanded by maintaining a parametric distribution over actions depending on $\theta$, similar to a tree adaptation of MPPI \citep{williams2017information} with a value function. This would significantly broaden the applicability of SoftTreeMax to important domains like robotics and continuous control.

Further important extensions include learning the forward model from interactions \citep{ha2018world, schrittwieser2020mastering} rather than using the true dynamics. Our analysis in Theorem~\ref{thm:bias} already provides theoretical guidance on the resulting gradient bias, but empirical validation with learned models would bridge the gap between model-based approaches like MuZero and policy gradient methods. Additionally, adapting the behavior policy $\pi_b$ over time as learning progresses could lead to more relevant exploration, though this requires careful management to maintain variance reduction benefits. Finally, the E-SoftTreeMax variant offers potential applications in risk-sensitive RL by naturally emphasizing exceptional trajectories through its reward exponentiation structure.

\section*{Impact Statement} 
This paper presents work whose goal is to advance the field of Machine Learning. There are many potential societal consequences of our work, none which we feel must be specifically highlighted here.

\section*{Reproducibility and Limitations}
The code for our implementation is available at \newline {\url{https://github.com/NVlabs/SoftTreeMax}}. We provide a docker file for setting up the environment and a README file with instructions on how to run both training and evaluation. The environment engine is an extension of Atari-CuLE \citep{dalton2020accelerating}, a CUDA-based Atari emulator that runs on GPU. Our usage of a GPU environment is both a novelty and a current limitation of our work.

There are additional limitations to consider. First, our analysis focuses on MDPs with finite state and action spaces, and while we discuss theoretical extensions to continuous domains, practical implementations would require careful design choices to manage the sampling process efficiently. Second, while SoftTreeMax can be applied to any environment, it provides most benefit in quasi-deterministic environments where accurate forward planning is possible. In highly stochastic environments, the exponential growth of the tree width with depth would present challenges even with pruning.

\bibliography{SoftTreeMaxBib}
\bibliographystyle{icml2025}

%%%%%%%%%%%%%%%%%%%%%%%%%%%%%%%%%%%%%%%%%%%%%%%%%%%%%%%%%%%%%%%%%%%%%%%%%%%%%%%
%%%%%%%%%%%%%%%%%%%%%%%%%%%%%%%%%%%%%%%%%%%%%%%%%%%%%%%%%%%%%%%%%%%%%%%%%%%%%%%
% APPENDIX
%%%%%%%%%%%%%%%%%%%%%%%%%%%%%%%%%%%%%%%%%%%%%%%%%%%%%%%%%%%%%%%%%%%%%%%%%%%%%%%
%%%%%%%%%%%%%%%%%%%%%%%%%%%%%%%%%%%%%%%%%%%%%%%%%%%%%%%%%%%%%%%%%%%%%%%%%%%%%%%
\newpage
\appendix
\onecolumn
\section*{Appendix}
\input{appendix.tex}

\end{document}

%% file: definitions.tex
\usepackage[normalem]{ulem}

\newcommand{\ignore}[1]{{}}

\newcommand{\cA}{\mathcal{A}}
\newcommand{\cS}{\mathcal{S}}
\newcommand{\ones}{\textbf{1}}
\newcommand{\tr}{\top}
\newcommand{\bR}{\mathbb{R}}
\newcommand{\bE}{\mathbb{E}}
\newcommand{\pid}{\pi_{d,\theta}}
\newcommand{\pidhat}{\hat{\pi}_{d,\theta}}
\newcommand{\pib}{\pi_b}

\newcommand{\Esd}{E_{s,d}}
\newcommand{\Csd}{C_{s,d}}
\newcommand{\Var}{\operatorname{Var}} 
\newcommand{\Sm}{B}
\newcommand{\Tr}{\operatorname{Tr}}

\newcommand{\hP}{\hat{P}}
\newcommand{\hr}{\hat{r}}

\newcommand{\treepol}{\text{SoftTreeMax}}
\newcommand{\txC}{\text{C}}
\newcommand{\txE}{\text{E}}
\newcommand{\treepolC}{\text{C-\treepol}}
\newcommand{\treepolE}{\text{E-\treepol}}

\newcommand{\h}[1]{\hat{#1}}

%% file: appendix.tex
\section{Proofs} \label{app:proofs}

\subsection{Proof of Lemma \ref{lem:var_bound} -- Bound on the policy gradient variance}
For any parametric policy $\pi_\theta$ and function $Q:\cS\times\cA \rightarrow \bR,$ 
\begin{align} \nonumber
\Var\left(\nabla_\theta\log \pi_\theta(a|s) Q(s,a)\right) \leq \max_{s,a} \left[Q(s,a)\right]^2   \max_s  \| \nabla_\theta\log \pi_\theta(\cdot|s)\|_F^2,
\end{align}
where $\nabla_\theta\log \pi_\theta(\cdot|s) \in \mathbb{R}^{A  \times \dim(\theta)}$ is a matrix whose $a$-th row is $\nabla_\theta\log \pi_\theta(a|s)^\top$.

\begin{proof}
The variance for a parametric policy $\pi_\theta$ is given as follows:
\begin{align*}
\Var\left(\nabla_\theta\log \pi_\theta(a|s) Q(a,s)\right) 
    =& \mathbb{E}_{s\sim d_{\pi_\theta},a\sim \pi_\theta(\cdot|s)}\left[\nabla_\theta\log \pi_\theta(a|s)^\top \nabla_\theta\log \pi_\theta(a|s) Q(s,a)^2\right] - \\ 
    & \quad\mathbb{E}_{s\sim d_{\pi_\theta},a\sim \pi_\theta(\cdot|s)}\left[\nabla_\theta\log \pi_\theta(a|s)Q(s,a)\right]^\top \mathbb{E}_{s\sim d_{\pi_\theta},a\sim \pi_\theta(\cdot|s)}\left[\nabla_\theta\log \pi_\theta(a|s)Q(s,a)\right],
\end{align*}
where $Q(s, a)$ is the currently estimated Q-function and $d_{\pi_\theta}$ is the discounted state visitation frequency induced by the policy $\pi_\theta$. Since the second term we subtract is always positive (it is of quadratic form $v^\top v$) we can bound the variance by the first term:

\begin{align*}
\Var\left(\nabla_\theta\log \pi_\theta(a|s) Q(a,s)\right) 
    \leq & \mathbb{E}_{s\sim d_{\pi_\theta},a\sim \pi_\theta(\cdot|s)} \left[\nabla_\theta\log \pi_\theta(a|s)^\top \nabla_\theta\log \pi_\theta(a|s) Q(s,a)^2\right] \\
    =&  \sum_{s}d_{\pi_\theta}(s) \sum_{a} \pi_\theta(a|s) \nabla_\theta\log \pi_\theta(a|s)^\top \nabla_\theta\log \pi_\theta(a|s) Q(s,a)^2 \\
    \leq & \max_{s,a} \left[ \left[Q(s,a)\right]^2 \pi_\theta(a| s)\right] \sum_{s}d_{\pi_\theta}(s) \sum_{a}\nabla_\theta\log \pi_\theta(a|s)^\top \nabla_\theta\log \pi_\theta(a|s) \\
    \leq & \max_{s,a} \left[Q(s,a)\right]^2  \max_s  \sum_{a}\nabla_\theta\log \pi_\theta(a|s)^\top \nabla_\theta\log \pi_\theta(a|s) \\    
    = & \max_{s,a}  \left[Q(s,a)\right]^2  \max_s  \| \nabla_\theta\log \pi_\theta(\cdot|s)\|_F^2.
\end{align*}
\end{proof}

\subsection{Proof of Lemma \ref{def:pidC} -- Vector form of \treepolC }
In vector form, \eqref{eq:polC} is given by
\begin{equation}
    \pid^{\txC}(\cdot|s) = \frac{ \exp\left[\beta \left(\Csd + P_s \left(P^{\pib}\right)^{d-1} \Theta \right)\right]}{\ones_A^\top \exp\left[\beta \left(\Csd + P_s \left(P^{\pib}\right)^{d-1} \Theta \right)\right]}, 
\end{equation}
where
\begin{equation} 
\Csd = \gamma^{-d} R_s + P_s \left[\sum_{h=1}^{d-1} \gamma^{h - d} \left(P^{\pib}\right)^{h - 1} \right] R_{\pib}.
\end{equation}

\begin{proof}
Consider the vector $\ell_{s, \cdot} \in \bR^{|\cA|}.$ Its expectation satisfies
%We calculate the expectation for all actions simultaneously:
\begin{align*}
\bE^{\pib} \ell_{s, \cdot}(d;\theta) &=  \bE^{\pib} \left[\sum_{t=0}^{d-1} \gamma^{t-d} r_t + \theta(s_d) \right] \\ 
% &= R_s + \sum_{t= 1}^{d-1} \gamma^t  \bE^{\pib} \left[r_t\right] + \gamma^d \bE^{\pib} \left[\theta(s_d) \right] \\
%
%&= R_s + \sum_{t= 1}^{d-1} \gamma^t  \sum_{s', a'}  \Pr(s_t = s'|s_0 = s, a_0 = \cdot) \pi_b(a'|s') r(s', a') + \gamma^d \bE^{\pib} \left[\theta(s_d) \right]\\
%
%&= R_s + \sum_{t= 1}^{d-1} \gamma^t  \sum_{s'}  \Pr(s_t = s'|s_0 = s, a_0 = \cdot) R_{\pib}(s') + \gamma^d \bE^{\pib} \left[\theta(s_d) \right]\\
%
&= \gamma^{-d} R_s + \sum_{t= 1}^{d-1} \gamma^{t-d}  P_s (P^{\pi_b})^{t - 1} R_{\pib} + P_s (P^{\pib})^{d-1}  \Theta.
\end{align*}
As required.
\end{proof}

\subsection{Proof of Lemma \ref{lem:analytic_gradC} -- Gradient of \treepolC}

The \treepolC\ gradient of dimension $A \times S$ is given by
\begin{align*} 
 \nabla_\theta\log \pid^{\txC}
 =\beta\left[I_{A} -  \ones_A {(\pid^{\txC}})^\top \right]P_s \left(P^{\pib}\right)^{d-1},
\end{align*}
where for brevity, we drop the $s$ index in the policy above, i.e., ${\pid^{\txC} \equiv \pid^{\txC}(\cdot|s).}$
\begin{proof}
The $(j, k)$-th entry of $\nabla_\theta\log \pid^{\txC}$ %i.e., the entry along the $j$-th action $a^j$ and $k$-th state $s^k,$ 
satisifes
% We calculate per index:
\begin{align*}
[\nabla_\theta\log \pid^{\txC}]_{j, k} 
&  = \frac{\partial \log (\pid^{\txC}(a^j|s))}{\partial \theta(s^k)} \\
&= \beta [P_s (P^{\pib})^{d-1}]_{j, k} - \frac{ \sum_a \left[\exp\left[\beta \left(\Csd + P_s \left(P^{\pib}\right)^{d-1} \Theta \right)\right]\right]_a \beta\left[P_s (P^{\pib})^{d-1}\right]_{a, k}}{\ones_A^\top \exp\left[\beta \left(\Csd +  P_s \left(P^{\pib}\right)^{d-1} \Theta \right)\right]} \\
&= \beta [P_s (P^{\pib})^{d-1}]_{j, k} - \beta \sum_a \pid^{\txC}(a|s)  \left[P_s (P^{\pib})^{d-1}\right]_{a, k} \\
&= \beta [P_s (P^{\pib})^{d-1}]_{j, k} - \beta \left[(\pid^{\txC})^\top P_s (P^{\pib})^{d-1}\right]_{k} \\
&= \beta [P_s (P^{\pib})^{d-1}]_{j, k} - \beta \left[\ones_A (\pid^{\txC})^\top P_s (P^{\pib})^{d-1}\right]_{j, k}.
\end{align*}
Moving back to matrix form, we obtain the stated result.
\end{proof}

\subsection{Proof of Theorem \ref{thm:rate_result} -- Exponential variance decay of \treepolC}\label{app:thm1proof}
The \treepolC\ policy gradient is bounded by
\begin{align*}
\Var \left(\nabla_\theta \log \pid^{\txC}(a|s) Q(s,a)\right)  \leq 2\frac{A^2 S^2\beta^2}{(1 - \gamma)^2} |\lambda_2(P^{\pib})|^{2(d-1)}.
\end{align*}
\begin{proof}
We use Lemma \ref{lem:var_bound} directly. First of all, it is know that when the reward is bounded in $[0,1]$, the maximal value of the Q-function is $\frac{1}{1-\gamma}$ as the sum as infinite discounted rewards. Next, we bound the Frobenius norm of the term achieved in Lemma \ref{lem:analytic_gradC}, by applying the eigen-decomposition on $P^{\pib}$:
\begin{equation}
P^{\pib} = \ones_S \mu^\top + \sum_{i=2}^S \lambda_i u_i v_i^\top,
\end{equation}
where $\mu$ is the stationary distribution of $P^{\pib}$, and $u_i$ and $v_i$ are left and right eigenvectors correspondingly.
\begin{align*}
\|\beta\left( I_{A,A} -  \ones_A \pi^\top \right)P_s (P^{\pib})^{d-1}\|_F  &= \beta\|\left( I_{A,A} -  \ones_A \pi^\top \right)P_s \left(\ones_S \mu^\top + \sum_{i=2}^S \lambda^{d-1}_i u_i v_i^\top \right)\|_F \\ 
\textit{(}P_s\textit{ is stochastic)} \quad & = \beta\|\left( I_{A,A} -  \ones_A \pi^\top \right) \left(\ones_A \mu^\top + \sum_{i=2}^S \lambda^{d-1}_i P_s u_i v_i^\top \right)\|_F \\
\textit{(projection nullifies }\ones_A \mu^\top\textit{)} \quad & = \beta\|\left( I_{A,A} -  \ones_A \pi^\top \right) \left(\sum_{i=2}^S \lambda^{d-1}_i P_s u_i v_i^\top \right)\|_F \\
\textit{(triangle inequality)} \quad& \leq \beta \sum_{i=2}^S\|\left( I_{A,A} -  \ones_A \pi^\top \right) \left( \lambda^{d-1}_i P_s u_i v_i^\top \right)\|_F \\
\textit{(matrix norm sub-multiplicativity)} \quad& \leq \beta |\lambda^{d-1}_2| \sum_{i=2}^S\| I_{A,A} -  \ones_A \pi^\top \|_F  \|P_s\|_F \|u_i v_i^\top \|_F \\
&= \beta |\lambda^{d-1}_2| (S-1) \| I_{A,A} -  \ones_A \pi^\top \|_F  \|P_s\|_F.
\end{align*}
Now, we can bound the norm $\| I_{A,A} -  \ones_A \pi^\top \|_F$ by direct calculation:
\begin{align}
\|I_{A,A} -  \ones_A \pi^\top\|_F^2 &= \Tr \left[ \left(I_{A,A} -  \ones_A \pi^\top \right) \left(I_{A,A} -  \ones_A \pi^\top \right)^\top \right] \\
&= \Tr \left[ I_{A,A} -  \ones_A \pi^\top - \pi \ones_A ^\top +  \pi^\top \pi \ones_A \ones_A^\top \right] \\
&= A -  1 - 1 + A \pi^\top \pi \\
&\leq 2A.
\end{align}
From the Cauchy-Schwartz inequality,
\begin{align*}
\|P_s\|_F^2 &= \sum_a \sum_s \left[[P_s]_{a,s} \right]^2 
= \sum_a \|[P_s]_{a, \cdot} \|_2^2 
 \leq \sum_a \|[P_s]_{a, \cdot} \|_1 \|[P_s]_{a, \cdot} \|_\infty \leq A.  
\end{align*}
So,
\begin{align*}
\Var \left(\nabla_\theta \log \pid^{\txC}(a|s) Q(s,a)\right)  & \leq \max_{s,a} \left[Q(s,a)\right]^2   \max_s  \| \nabla_\theta\log {\pid^{\txC}}(\cdot|s)\|_F^2 \\
&\leq \frac{1}{(1-\gamma)^2} \|\beta\left( I_{A,A} -  \ones_A \pi^\top \right)P_s (P^{\pib})^{d-1}\|^2_F \\
&\leq \frac{1}{(1-\gamma)^2}\beta^2  |\lambda_2(P^{\pib})|^{2(d-1)} S^2 (2A^2),
\end{align*}
which obtains the desired bound.
\end{proof}

\subsection{A lower bound on \treepolC\ gradient (result not in the paper)}
For completeness we also supply a lower bound on the Frobenius norm of the gradient. Note that this result does not translate to the a lower bound on the variance since we have no lower bound equivalence of Lemma \ref{lem:var_bound}. 
\begin{lemma}
The Frobenius norm on the gradient of the policy is lower-bounded by:
\begin{equation}
\| \nabla_\theta\log {\pid^{\txC}}(\cdot|s)\|_F \geq C \cdot \beta |\lambda_2(P^{\pib})|^{(d-1)}.
\end{equation}
\end{lemma}
\begin{proof}
We begin by moving to the induced $l_2$ norm by norm-equivalence:
\begin{equation*}
\|\beta\left( I_{A,A} -  \ones_A \pi^\top \right)P_s (P^{\pib})^{d-1}\|_F \geq \|\beta\left( I_{A,A} -  \ones_A \pi^\top \right)P_s (P^{\pib})^{d-1}\|_2.
\end{equation*}
Now, taking the vector $u$ to be the eigenvector of the second eigenvalue of $P^{\pib}$:
\begin{align*}
\|\beta\left( I_{A,A} -  \ones_A \pi^\top \right)P_s (P^{\pib})^{d-1}\|_2 &\geq \|\beta\left( I_{A,A} -  \ones_A \pi^\top \right)P_s (P^{\pib})^{d-1} u\|_2 \\
&= \beta\|\left( I_{A,A} -  \ones_A \pi^\top \right)P_s  u\|_2 \\
&= \beta |\lambda_2(P^{\pib})|^{(d-1)} \|\left( I_{A,A} -  \ones_A \pi^\top \right)P_s  u\|_2.
\end{align*}
Note that even though $P_s u$ can be $0$, that is not the common case since we can freely change $\pib$ (and therefore the eigenvectors of $P^{\pib}$).
\end{proof}

\subsection{Proof of Lemma \ref{def:pidE} -- Vector form of \treepolE}
For $d\geq1$, \eqref{eq:polE} is given by
\begin{equation}
    \pid^{\txE}(\cdot|s)= \frac{\Esd \exp(\beta \Theta)}{1_A^\top  \Esd \exp(\beta  \Theta)},
\end{equation}
where 
%
% \begin{equation}
% \Esd = D\left(\beta \exp(R_s)\right) P_s  \prod_{h=1}^{d-1} \left(D\left(\exp(\beta \gamma^h R_{\pib})\right) P^{\pib} \right).
% \end{equation}
%
\begin{equation}
\Esd = P_s  \prod_{h=1}^{d-1} \left(D\left(\exp[\beta \gamma^{h-d} R]\right) P^{\pib} \right)
\end{equation}
with $R$ being the $|S|$-dimensional vector whose $s$-th coordinate is $r(s).$
%
% The final expression is
% %
% \begin{equation}
%     \ell_s (d; \theta) = \exp [\beta r(s)] P_s \prod_{t = 1}^{d - 1} [ D(\exp[\beta \gamma^t R]) P^{\pib}]  \exp (\beta \gamma^d \Theta).
% \end{equation}
% %
% \begin{equation}
%     \pi_b(\cdot|s_1)^\tr D(\exp(\beta \gamma r(s_1, \cdot))) P_{s_1} \exp[\beta \gamma^2 \Theta ] 
% \end{equation}
%
\begin{proof}
Recall that 
\begin{equation}
    \ell_{s, a}(d; \theta) = \gamma^{-d} \left[r(s) + \sum_{t = 1}^{d - 1} \gamma^t r(s_t) + \gamma^d \theta (s_d)\right].
\end{equation}
and, hence, 
\begin{equation}
    \exp[\beta \ell_{s, a}(d; \theta)] = \exp \left[ \beta\gamma^{-d}\left(r(s) + \sum_{t = 1}^{d - 1} \gamma^t r(s_t) + \gamma^d \theta (s_d) \right)\right].
\end{equation}
Therefore, 
\begin{align}
    \bE [\exp \beta \ell_{s, a}(d; \theta)] = {} & \bE \left[ \exp \left[ \beta\gamma^{-d}\left(r(s) + \sum_{t = 1}^{d - 1} \gamma^t r(s_t) \right)\right]\bE \left[\exp \left[ \beta\left(  \theta (s_d) \right)\right] \middle|s_1, \ldots, s_{d  - 1}  \right]\right] \\
    = {} & \bE \left[ \exp \left[ \beta\gamma^{-d}\left(r(s) + \sum_{t = 1}^{d - 1} \gamma^t r(s_t) \right)\right] P^{\pib}(\cdot|s_{d - 1})\right] \exp (\beta \Theta) \\
    = {} &  \bE \left[ \exp \left[ \beta\gamma^{-d}\left(r(s) + \sum_{t = 1}^{d - 2} \gamma^t r(s_t) \right)\right] \exp[\beta \gamma^{ - 1} r(s_{d - 1}) ] P^{\pib}(\cdot|s_{d - 1})\right] \exp (\beta  \Theta).
\end{align}
By repeatedly using iterative conditioning as above, the desired result follows. Note that $\exp(\beta \gamma^{-d} r(s))$ does not depend on the action and is therefore cancelled out with the denominator. 
\end{proof}

\subsection{Proof of Lemma~\ref{lem:analytic_gradE} -- Gradient of \treepolE}

The \treepolE\ gradient  of dimension $A \times S$ is given by
\begin{align*}
\nabla_\theta\log \pid^{\txE} =    
     \beta\left[I_{A} - \ones_A (\pid^{\txE})^\top \right]\frac{ D\left(\pid^{\txE}\right)^{-1} \Esd D(\exp(\beta \Theta))}{{\bf{1}}^{\top}_A \Esd \exp(\beta \Theta)}, 
\end{align*}
where for brevity, we drop the $s$ index in the policy above, i.e., ${\pid^{\txE} \equiv \pid^{\txE}(\cdot|s).}$

\begin{proof}
The $(j, k)$-th entry of $\nabla_\theta\log \pid^{\txE}$ %i.e., the entry along the $j$-th action $a^j$ and $k$-th state $s^k,$ 
satisfies
% We calculate per index:
\begin{align*}
[\nabla_\theta\log \pid^{\txE}]_{j, k} = {} & \frac{\partial \log (\pid^{\txE}(a^j|s))}{\partial \theta(s^k)} \\
= {} & \frac{\partial}{\partial \theta(s^k)} \left(\log [(E_{s, d})_{j}^\tr \exp (\beta  \Theta)] - \log [\ones_A^\tr E_{s, d} \exp(\beta  \Theta) ]\right)\\
= {} & \frac{\beta  (E_{s, d})_{j, k} \exp(\beta  \theta(s^k))  }{(E_{s, d})_{j}^\tr \exp (\beta  \Theta)} - \frac{\beta  \ones_A^\tr E_{s, d} e_k \exp(\beta  \theta(s^k))}{\ones_A^\tr E_{s, d} \exp(\beta  \Theta)} \\
= {} & \frac{\beta  (E_{s, d} e_k  \exp(\beta  \theta(s^k)) )_j }{(E_{s, d})_{j}^\tr \exp (\beta  \Theta)} - \frac{\beta  \ones_A^\tr E_{s, d} e_k \exp(\beta  \theta(s^k))}{\ones_A^\tr E_{s, d} \exp(\beta  \Theta)} \\
= {} & \beta  \left[\frac{e_j^\tr}{e_j^\tr E_{s, d} \exp(\beta  \Theta)} - \frac{\ones_A^\tr}{\ones_A^\tr E_{s, d} \exp(\beta  \Theta)} \right] E_{s, d} e_k \exp (\beta \theta(s^k)).
\end{align*}
Hence, 
\begin{align*}
[\nabla_\theta\log \pid^{\txE}]_{\cdot, k} = \beta  \left[D(E_{s,d} \exp(\beta \Theta))^{-1}  - (\ones_A^\tr E_{s, d} \exp(\beta  \Theta))^{-1} \ones_A \ones_A^\tr \right] E_{s, d} e_k \exp (\beta  \theta(s^k))
\end{align*}
From this, it follows that 
\begin{equation}
\label{e:Grad.Useful.Form}
\nabla_\theta\log \pid^{\txE} = \beta  \left[D\left(\pid^{\txE}\right)^{-1} - \ones_A \ones_A^\tr \right]  \frac{ E_{s, d} D( \exp (\beta  \Theta))}{\ones_A^\tr E_{s, d} \exp(\beta \Theta)}.
\end{equation}

The desired result is now easy to see. 
\end{proof}

\subsection{Proof of Theorem~\ref{thm:rate_result2} --- Exponential variance decay of \treepolE} \label{app:proof_thm2}

There exists $\alpha \in \left(0,1\right)$ such that, for any function $Q: \cS \times \cA \rightarrow \bR,$
\begin{equation*}
\Var\left(\nabla_\theta \log \pid^{\txE}(a|s) Q(s,a)\right) \in \mathcal{O}\left(\beta^2  \alpha^{2d}\right).
\end{equation*}
If all rewards are equal ($r\equiv \text{const}$), then $\alpha = |\lambda_2(P^{\pib})|$.

\begin{proof}[Proof outline]
Recall that thanks to Lemma~\ref{lem:var_bound}, we can bound the PG variance using a direct bound on the gradient norm. The definition of the induced norm is
\begin{equation*}
    \|\nabla_\theta\log \pid^{\txE}\| = \max_{z: \|z\| = 1} \|\nabla_\theta\log \pid^{\txE} z\|,
\end{equation*}
with $\nabla_\theta\log \pid^{\txE}$ given in Lemma~\ref{lem:analytic_gradE}.
Let $z \in \bR^S$ be an arbitrary vector such that $\|z\| = 1$. Then, 
$
    z = \sum_{i = 1}^S c_i z_i,
$
% \end{equation*}
where $c_i$ are scalar coefficients and $z_i$ are vectors spanning the $S$-dimensional space. In the full proof, we show our specific choice of $z_i$ and prove they are linearly independent given that choice. We do note that $z_1=\ones_S.$

The first part of the proof relies on the fact that $(\nabla_\theta\log \pid^{\txE}) z_1 = 0.$ This is easy to verify using Lemma~\ref{lem:analytic_gradE} together with \eqref{eq:SoftTreeMax_matrix2}, and because $\left[I_{A} - \ones_A (\pid^{\txE})^\top \right]$ is a projection matrix whose null-space is spanned by $\ones_S.$  Thus, $$\nabla_\theta\log \pid^{\txE} z = \nabla_\theta\log \pid^{\txE} \sum_{i=2}^S c_i z_i.$$

In the second part of the proof, we focus on $\Esd$ from \eqref{eq:SoftTreeMax_matrix2}, which appears within $\nabla_\theta\log \pid^{\txE}.$ Notice that $\Esd$ consists of the product 
$\prod_{h=1}^{d-1} \left(D\left(\exp(\beta \gamma^{h-d} R\right) P^{\pib} \right).$
Even though the elements in this product are not stochastic matrices, in the full proof we show how to normalize each of them to a stochastic matrix $B_h.$ We thus obtain that
\begin{equation*}
\Esd = P_s  D(M_1) \prod_{h=1}^{d-1} B_h,
\end{equation*}
where $M_1 \in \bR^S$ is some strictly positive vector. Then, we can apply a result by \citet{mathkar2016nonlinear}, which itself builds on \citep{chatterjee1977towards}. The result states that the product of stochastic matrices $\prod_{h=1}^{d-1}B_h$ of our particular form converges exponentially fast to a matrix of the form $\ones_S \mu^\tr$ s.t. $\|\ones_S \mu^\tr - \prod_{h=1}^{d-1}B_h\| \leq C \alpha^d$ for some constant $C.$ 

Lastly,   $\ones_S \mu_{\pib}^\top$ gets canceled due to our choice of $z_i,~i=2,\dots,S.$ This observation along with the above fact that the  remainder decays then shows that $\nabla_\theta\log \pid^{\txE} \sum_{i=2}^S z_i = \mathcal{O}(\alpha^d),$ which gives the desired result.
\end{proof}

\begin{proof}[Full technical proof]
%
%\gug{Let $d \geq 2.$}
%
Let $d \geq 2.$ Recall that 
\begin{equation}
    \Esd = P_s  \prod_{h=1}^{d-1} \left(D\left(\exp[\beta \gamma^{h-d} R]\right) P^{\pib} \right),
\end{equation}
and that $R$ refers to the $S$-dimensional vector whose $s$-th coordinate is $r(s).$

Define
\begin{equation}
    \Sm_i=\begin{cases}
			P^{\pib} & \text{ if } i = d - 1,\\
            D^{-1}(P^{\pib} M_{i + 1}) P^{\pib} D (M_{i + 1}) & \text{ if } i = 1, \ldots, d - 2,
		 \end{cases}
\end{equation}
% \begin{equation}
%     \Sm_{d-1} = P^{\pib} \mbox{ and } B_i = D^{-1}(P^{\pib} M_{i + 1}) P^{\pib} D (M_{i + 1}) \mbox{ for } i = 1, \ldots, d - 2,
% \end{equation}
and the vector
\begin{equation}
\label{def: Mi}
    M_i = 
    \begin{cases}
    \exp(\beta\gamma^{-1} R) & \text{ if } i = d - 1, \\
    \exp (\beta \gamma^{i-d} R) \circ P^{\pib} M_{i + 1} & \text{ if } i = 1, \ldots, d - 2,
    \end{cases}
\end{equation}
where $\circ$ denotes the element-wise product.
Then,
\begin{equation}
\label{eq: Esd decompsition}
    E_{s, d} = P_s D(M_1) \prod_{i = 1}^{d - 1} \Sm_i. 
\end{equation}
It is easy to see that each $B_i$ is a row-stochastic matrix, i.e., all entries are non-negative and $B_i \ones_S = \ones_S.$

Next, we prove that all non-zeros entries of $\Sm_i$ are bounded away from $0$ by a constant. This is necessary to apply the next result from \cite{chatterjee1977towards}. The $j$-th coordinate of $M_i$ satisfies
\begin{equation}
    (M_i)_j =  \exp[\beta \gamma^{i-d} R_j]\sum_{k} [P^{\pib}]_{j,k}(M_{i+1})_k \leq \|\exp[\beta \gamma^{i-d} R]\|_\infty \|M_{i+1}\|_\infty.
\end{equation}
Separately, observe that $\|M_{d - 1}\|_\infty \leq \|\exp(\beta \gamma^{ - 1} R)\|_\infty.$ Plugging these relations in \eqref{def: Mi} gives
%
%backwards from $i = d - 2$ to $1$ together with the bound above, we get
\begin{equation}
    \|M_1\|_\infty \leq \prod_{h=1}^{d-1}\|\exp[\beta \gamma^{h-d} R]\|_\infty=\prod_{h=1}^{d-1}\|\exp[\beta \gamma^{-d} R]\|^{\gamma^h}_\infty = \|\exp[\beta  \gamma^{-d} R]\|_\infty ^{\sum_{h=1}^{d-1} \gamma^h} \leq \|\exp[\beta \gamma^{-d} R]\|_\infty^{\frac{1}{1 - \gamma}}.
\end{equation}

Similarly, for every $1 \leq i \leq d-1,$ we have that
\begin{equation}
    \|M_i\|_\infty \leq \prod_{h=i}^{d-1}\|\exp[\beta  \gamma^{-d} R]\|^{\gamma^h}_\infty\leq \|\exp[\beta \gamma^{-d} R]\|_\infty^{\frac{1}{1 - \gamma}}.
\end{equation}

The $jk$-th entry of $B_i = D^{-1}(P^{\pib} M_{i+1}) P^{\pib} D(M_{i+1})$ is
\begin{equation}
    (B_i)_{jk} = \frac{P^{\pib}_{jk} [M_{i+1}]_k}{\sum_{\ell = 1}^{|S|} P^{\pib}_{j \ell} [M_{i+1}]_{\ell}} \geq \frac{P^{\pib}_{jk}}{\sum_{\ell = 1}^{|S|} P^{\pib}_{j \ell} [M_{i+1}]_{\ell}} \geq \frac{P^{\pib}_{jk}}{\| \exp[\beta \gamma^{-d} R]\|_\infty^{\frac{1}{1 - \gamma}}}.
\end{equation}
Hence, for non-zero $P^{\pib}_{jk}$, the entries are bounded away from zero by the same. We can now proceed with applying the following result.

Now, by \citep[Theorem~5]{chatterjee1977towards} (see also (14) in \citep{mathkar2016nonlinear}), $    \lim_{d \to \infty} \prod_{i = 1}^{d - 1} \Sm_i$ exists and is of the form $\ones_S \mu^\tr$ for some probability vector $\mu.$ Furthermore, there is some $\alpha \in (0, 1)$ such that $\varepsilon(d):= \left(\prod_{i = 1}^{d - 1} \Sm_i\right) - \ones_S \, \mu^\tr$ satisfies
\begin{equation}
    \left\|\varepsilon(d) \right\| = O(\alpha^d).
\end{equation}
Pick linearly independent vectors $w_2, \ldots, w_S$ such that 
\begin{equation}
\label{eq: wi condtions}
    \mu^\tr w_i = 0 \mbox{ for } i = 2, \ldots, d.
\end{equation}
Since  $\sum_{i = 2}^S \alpha_i w_i$ is perpendicular to $\mu$ for any $\alpha_2, \ldots \alpha_S$ and because $\mu^\tr \exp(\beta  \Theta) > 0,$ there exists no choice of $\alpha_2, \ldots, \alpha_S$ such that $\sum_{i = 2}^S \alpha_i w_i = \exp(\beta  \Theta).$ Hence, if we let  $z_1 = \ones_S$ and $z_i = D(\exp(\beta  \Theta))^{-1} w_i$ for $i = 2, \ldots, S,$ then it follows that $\{z_1, \ldots, z_S\}$ is linearly independent. In particular, it implies that $\{z_1, \ldots, z_S\}$ spans $\bR^{S}.$

Now consider an arbitrary unit norm vector $
z := \sum_{i = 1}^S c_i z_i \in \bR^{S}$ s.t. $\|z\|_2=1.$ Then, 
\begin{align}
\nabla_\theta \log \pid^{\txE} z  = {} & \nabla_\theta \log \pid^{\txE} \sum_{i = 2}^S c_i z_i \label{e:cancel.Allones}\\
= {} & \beta  \left[I_{A} - \ones_A (\pid^{\txE})^\top \right]  \frac{D\left(\pid^{\txE}\right)^{-1} E_{s, d} D( \exp (\beta  \Theta))}{\ones_A^\tr E_{s, d} \exp(\beta  \Theta)} \sum_{i = 2}^S c_i z_i \label{e:Defn.Substitution}\\
= {} & \beta \left[I_{A} - \ones_A (\pid^{\txE})^\top \right]\frac{D\left(\pid^{\txE}\right)^{-1} E_{s, d}}{\ones_A^\tr E_{s, d}\exp(\beta  \Theta)}  \sum_{i = 2}^S c_i w_i \label{e:z.w.swap}\\
= {} & \beta \left[I_{A} - \ones_A (\pid^{\txE})^\top \right] \frac{D\left(\pid^{\txE}\right)^{-1} \left[ \ones_S \mu^\tr + \varepsilon(d)\right] }{\ones_A^\tr E_{s, d}\exp(\beta  \Theta)}  \sum_{i = 2}^S c_i w_i \\
= {} & \beta \left[I_{A} - \ones_A (\pid^{\txE})^\top \right]  \frac{D\left(\pid^{\txE}\right)^{-1} \varepsilon(d) }{\ones_A^\tr E_{s, d}\exp(\beta  \Theta)}   \sum_{i = 2}^S c_i w_i \label{e:mu.cancellation} \\
= {} &\beta \left[I_{A} - \ones_A (\pid^{\txE})^\top \right]  \frac{ D\left(\pid^{\txE}\right)^{-1}\varepsilon(d) D(\exp(\beta  \Theta)) }{\ones_A^\tr E_{s, d}\exp(\beta  \Theta)} (z - c_1 \ones_S), \label{e:z.w.reswap}
\end{align}
where \eqref{e:cancel.Allones} follows from the fact that $\nabla_\theta \log \pid^{\txE} z_1 = \nabla_\theta \log \pid^{\txE} \ones_S = 0,$ \eqref{e:Defn.Substitution} follows from Lemma~\ref{lem:analytic_gradE}, \eqref{e:z.w.swap} holds since $z_i = D(\exp(\beta  \Theta))^{-1} w_i,$ \eqref{e:mu.cancellation} because $\mu$ is perpendicular $w_i$ for each $i,$ while \eqref{e:z.w.reswap} follows by reusing $z_i = D(\exp(\beta  \Theta))^{-1} w_i$ relation along with the fact that $z_1 = \ones_S.$

From \eqref{e:z.w.reswap}, it follows that 
\begin{align}
    \|\nabla_\theta \log \pid^{\txE} z \| \leq {} &  \beta \|\varepsilon(d) \| \left\|\left[I_{A} - \ones_A (\pid^{\txE})^\top \right]  \frac{ D\left(\pid^{\txE}\right)^{-1} }{\ones_A^\tr E_{s, d}\exp(\beta  \Theta)} \right\| \|D(\exp(\beta  \Theta)) \| \, \|z - c_1 \ones_S\| \\
    \leq {} & \beta  \alpha^d (\|I_A\| + \|\ones_A (\pid^{\txE})^\top \|) \left\|  \frac{ D\left(\pid^{\txE}\right)^{-1} }{\ones_A^\tr E_{s, d}\exp(\beta  \Theta)}  \right\|\exp (\beta  \max_s \theta(s))   \|z - c_1 \ones_S\|  \\
    \leq {} & \beta \alpha^d (1 +  \sqrt{A}) \left\|  \frac{ D\left(\pid^{\txE}\right)^{-1} }{\ones_A^\tr E_{s, d}\exp(\beta  \Theta)}  \right\|\exp (\beta  \max_s \theta(s))   \|z - c_1 \ones_S\| \\
    \leq {} & \beta  \alpha^d (1 +  \sqrt{A}) \left\|  D^{-1}(E_{s, d} \exp(\beta  \Theta)) \right\|\exp (\beta  \max_s \theta(s))   \|z - c_1 \ones_S\| \\
    \leq {} & \beta  \alpha^d (1 +  \sqrt{A}) \frac{1}{\min_s [E_{s, d} \exp(\beta  \Theta]_s} \exp (\beta \max_s \theta(s))   \|z - c_1 \ones_S\| \\
    \leq {} & \beta  \alpha^d (1 +  \sqrt{A}) \frac{\exp (\beta  \max_s \theta(s)) }{\exp(\beta \min_s \theta(s))  \min_s |M_1|}  \|z - c_1 \ones_S\| \\
    \leq {} & \beta  \alpha^d (1 +  \sqrt{A}) \frac{\exp (\beta  \max_s \theta(s)) }{\exp(\beta \min_s \theta(s))  \exp (\beta \min_s r(s))}  \|z - c_1 \ones_S\| \\
    \leq {} & \beta  \alpha^d (1 +  \sqrt{A}) \exp(\beta [\max_s \theta(s) - \min_s \theta(s) - \min_s r(s)]) \|z - c_1 \ones_S\|.
\end{align}

Lastly, we prove that $\|z - c_1 \ones_S \|$ is bounded independently of $d.$ First, denote by $c=(c_1,\dots,c_S)^\top$ and $\tilde{c} = (0,c_2,\dots,c_S)^\top.$ Also, denote by $Z$ the matrix with $z_i$ as its $i$-th column. Now,
\begin{align}
\|z - c_1 \ones_S \| &= \| \sum_{i=2}^S c_i z_i \| \\
&= \| Z \tilde{c} \| \\
&\leq \|Z \| \| \tilde{c} \| \\
& \leq \|Z \| \| c \| \\
&= \|Z \| \| Z^{-1} z\| \\
& \leq \|Z \| \| Z^{-1}\| , \label{eq: z z inverse}
\end{align}
where the last relation is due to $z$ being a unit vector. All matrix norms here are $l_2$-induced norms. 

Next, denote by $W$ the matrix with $w_i$ in its $i$-th column. Recall that in \eqref{eq: wi condtions} we only defined $w_2,\dots,w_S.$ We now set $w_1= \exp(\beta \Theta)$. Note that $w_1$ is linearly independent of $\{w_2,\dots,w_S\}$ because of \eqref{eq: wi condtions} together with the fact that $\mu^\top w_1 > 0.$ We can now express the relation between $Z$ and $W$ by $Z=D^{-1}(\exp(\beta\Theta)) W.$ 
Substituting this in \eqref{eq: z z inverse}, we have 
\begin{align}
    \|z - c_1 \ones_S \| & \leq \|D^{-1}(\exp(\beta\Theta)) W \| 
    \| W^{-1} D(\exp(\beta\Theta)) \| \\
& \leq \| W \| \| W^{-1} \| \|D(\exp(\beta\Theta))\| \|D^{-1}(\exp(\beta\Theta))\| . \label{eq: 2nd}
\end{align}
It further holds that
\begin{equation}
    \|D(\exp(\beta\Theta))\| \leq \max_s \exp\left( \beta  \theta(s)\right) \leq \max\{1, \exp[\beta \max_s \theta(s)] )\}, \label{eq: 3rd}
\end{equation}
where the last relation equals $1$ if $\theta(s)<0$ for all $s.$ Similarly, 
\begin{equation}
    \|D^{-1}(\exp(\beta\Theta))\| \leq \frac{1}{\min_s \exp\left( \beta  \theta(s)\right)} \leq \frac{1}{\min\{1, \exp[\beta \min_s \theta(s)] )\}}. \label{eq: 4th}
\end{equation}
Furthermore, by the properties of the $l_2$-induced norm,
\begin{align}
    \| W \|_2 &\leq \sqrt{S}\|W\|_1 \\
    &= \sqrt{S} \max_{1 \leq i \leq S} \|w_i\|_1 \\
    & = \sqrt{S} \max \{\exp(\beta  \Theta), \max_{2 \leq i \leq S}\|w_i\|_1\} \\
    & \leq \sqrt{S} \max\{1, \exp[\beta \max_s \theta(s)], \max_{2 \leq i \leq S}\|w_i\|_1 )\}. \label{eq: 5th}
\end{align}

Lastly,
\begin{align}
    \| W ^{-1}\| & = \frac{1}{\sigma_{\min}(W)} \\
    & \leq \left( \prod_{i=1}^{S-1} \frac{\sigma_{\max}(W)}{\sigma_i(W)}\right)\frac{1}{\sigma_{\min}(W)} \\
    & = \frac{\left(\sigma_{\max}(W)\right)^{S-1}}{\prod_{i=1}^{S}\sigma_i(W)} \\
    & = \frac{\|W\|^{S-1}}{|\det(W)|}. \label{eq: determinant}
\end{align}
The determinant of $W$ is a sum of products involving its entries. To upper bound \eqref{eq: determinant} independently of $d,$ we lower bound its denominator by upper and lower bounds on the entries $[W]_{i,1}$ that are independent of $d,$ depending on their sign:
\begin{equation}
    \min\{1, \exp[\beta \min_s \theta(s)] )\} \leq [W]_{i,1} \leq \max\{1, \exp[\beta \max_s \theta(s)] )\}.
\end{equation}
Using this, together with \eqref{eq: z z inverse}, \eqref{eq: 2nd}, \eqref{eq: 3rd}, \eqref{eq: 4th}, and \eqref{eq: 5th}, we showed that $\|z - c_1 \ones_S\|$ is upper bounded by a constant independent of $d.$ This concludes the proof.
\end{proof}

\subsection{Bias Estimates} \label{app:biasest}

\begin{lemma} \label{lem: power decomposition}
    For any matrix $A$ and $\h{A},$ 
    $$\h{A}^k - A^k = \sum_{h = 1}^k \h{A}^{h - 1} (\h{A} - A)A^{k - h}.$$ 
\end{lemma}
\begin{proof}
    The proof follows from first principles:
    \begin{align}
    \sum_{h = 1}^k \h{A}^{h - 1} (\h{A} - A)A^{k - h} &= \sum_{h = 1}^k \h{A}^{h - 1} \h{A} A^{k - h} - \sum_{h = 1}^k \h{A}^{h - 1} A A^{k - h} \\
    &= \sum_{h = 1}^k \h{A}^{h}  A^{k - h} - \sum_{h = 1}^k \h{A}^{h - 1} A^{k - h + 1} \\ 
    &= \h{A}^k - A^k + \sum_{h = 1}^{k-1} \h{A}^{h}  A^{k - h} - \sum_{h = 2}^k \h{A}^{h - 1}  A^{k - h + 1} \\
    &= \h{A}^k - A^k.
    \end{align}
    
\end{proof}

Henceforth, $\|\cdot\|$ will refer to $\|\cdot\|_\infty,$ i.e. the induced infinity norm. Also, for brevity, we denote $\pid^{\txC}$ and $\pidhat^{\txC}$ by $\pi_\theta$ and $\h{\pi}_\theta,$ respectively. Similarly, we use $d_{\pi_\theta}$ and $d_{\h{\pi}_\theta}$ to denote $d_{\pid^{\txC}}$ and $d_{\pidhat^{\txC}}.$ As for the induced norm of the matrix $P$ and its perturbed counterpart $\h{P},$ which are of size $S \times A \times S,$ we slightly abuse notation and denote $\|P-\h{P}\| = \max_s\{\|P_s - \h{P}_s\|\},$ where $P_s$ is as defined in Section~\ref{sec:preliminaries}.
\begin{definition} \label{def: epsilon}
    Let $\epsilon$ be the maximal model mis-specification, i.e.,
$\max\{\|P - \hP\|, \|r- \hr\|\} = \epsilon.$
\end{definition}

\begin{lemma}
    \label{lem: linear deviations}
    Recall the definitions of $R_s, P_s, R_{\pi_b}$ and $P^{\pi_b}$ from Section~\ref{sec:preliminaries}, and respectively denote their perturbed counterparts by $\h{R}_s, \h{P}_s, \h{R}_{\pi_b}$ and $\h{P}^{\pi_b}$. Then, for $\epsilon$ defined in Definition~\ref{def: epsilon}, 
\begin{equation} \label{eq: bound on pertubations}
\max\{\|R_s - \h{R}_s\|,\|P_s - \h{P}_s\|,\|R_{\pi_b} - \h{R}_{\pi_b}\|,\|P^{\pi_b} - \h{P}^{\pi_b}\|\} = O(\epsilon).    
\end{equation}
\end{lemma}
\begin{proof}
    The proof follows easily from the fact that the differences above are convex combinations of $P-\h{P}$ and $r-\h{r}.$
\end{proof}
% linear transformations

\begin{lemma}
\label{lem:pol.Diff.Bd}
 Let $\pi_\theta$ be as in \eqref{eq:SoftTreeMax_matrix}, and let $\h{\pi}_\theta$ also be defined as in \eqref{eq:SoftTreeMax_matrix}, but with $R_s, P_s, P^{\pib}$ replaced by their perturbed counterparts $\hat{R}_s, \hat{P}_s, \hat{P}^{\pib}$ throughout. 
 % We will interpret the latter as the policy one would have obtained if they had access only to the simulator model.
 Then, 
\begin{equation}
    \|\pid^{\txC} - \pidhat^{\txC}\| = O(\beta d \gamma^{-d} \epsilon).
\end{equation}
\end{lemma}
\begin{proof}

To prove the desired result, we work with \eqref{eq:SoftTreeMax_matrix} to bound the error between $R_s, P_s, P^{\pib},R_{\pib}$ and their perturbed versions. 

First, we apply Lemma~\ref{lem: power decomposition}  together with Lemma~\ref{lem: linear deviations} to obtain that $\|(P^{\pi_b})^k - (\h{P}^{\pi_b})^k\|=O(k\epsilon).$ Next, denote by $M$ the argument in the exponent in \eqref{eq:SoftTreeMax_matrix}, i.e. $$M := \beta[C_{s, d} +  P_s (P^{\pi_b})^{d - 1} \Theta ].$$ Similarly, let $\h{M}$ be the corresponding perturbed sum that relies on $\h{P}$ and $\h{r}$.
Combining the bounds from Lemma~\ref{lem: linear deviations}, and using the triangle inequality, we have that $\|\h{M} - M\| = O(\beta d \gamma^{-d} \epsilon)$. The factor $\gamma^{-d}$ appears because $C_{s,d}$ includes the term $\gamma^{-d}$ as shown in Lemma~\ref{def:pidC}.

Eq.~\eqref{eq:SoftTreeMax_matrix} states that the C-SoftTreeMax policy in the true environment is $\pi_\theta = \exp(M)/(1^\top \exp(M))$. Similarly define $\h{\pi}_\theta$ using $\h{M}$ for the approximate model. Then,
	$$
		\h{\pi}_\theta = (\pi_\theta \circ \exp(\h{M} - M)) 1^\top \exp(M)/(1^\top \exp(\h{M})),
	$$
	where $\circ$ denotes element-wise multiplication. 
 Using the above relation, we have that $\|\h{\pi}_\theta - \pi_\theta\| = \|\pi_\theta\| \|\frac{\exp(\h{M} - M) 1^\top \exp(M)}{1^\top \exp(\h{M})} - 1\|.$ Using the relation $|e^x - 1| = O(x)$ as $x \to 0,$ the desired result follows.

\end{proof}

\begin{theorem}
Let $\epsilon$ be as in Definition~\ref{def: epsilon}. Further let $\pidhat^{\txC}$ being the corresponding approximate policy as given in Lemma~\ref{def:pidC}. Then, the policy gradient bias is bounded by 
\begin{equation}
    \left\|\frac{\partial}{\partial \theta}  \left(\nu^\top V^{\pi_\theta}\right) -  \frac{\partial}{\partial \theta}  \left(\nu^\top V^{\h{\pi}_\theta}\right) \right\| = \mathcal{O}\left(\frac{1}{(1-\gamma)^2} S \beta^2 d \gamma^{-d} \epsilon\right).
\end{equation}
\end{theorem}

We first provide a proof outline for conciseness, and only after it the complete proof. 

\begin{proof}[Proof outline]
%

%
% \begin{equation} 
% \label{eq: bound on pertubations outline}
%     \max\{\|R_s - \h{R}_s\|,\|P_s - \h{P}_s\|,\|R_{\pi_b} - \h{R}_{\pi_b}\|,\|P^{\pi_b} - \h{P}^{\pi_b}\|\} = O(\epsilon).    
% \end{equation}
%
First, we prove that $\max\{\|R_s - \h{R}_s\|,\|P_s - \h{P}_s\|,\|R_{\pi_b} - \h{R}_{\pi_b}\|,\|P^{\pi_b} - \h{P}^{\pi_b}\|\} = \mathcal{O}(\epsilon).$ This follows from the fact that the differences above are suitable convex combinations of either the rows of $P-\h{P}$ or $r-\h{r}.$ 
% Next, using some linear algebra, we prove in Lemma~\ref{lem: power decomposition} that, for any two matrices $A$ and $\h{A},$ the relation $\h{A}^k - A^k = \sum_{h = 1}^k \h{A}^{h - 1}(\h{A} - A) A^{k - h}$ holds. 
We use the above observation along with the definitions of $\pid^{\txC}$ and $\pidhat^{\txC}$ given in \eqref{eq:SoftTreeMax_matrix} to show that $\|\pid^{\txC} - \pidhat^{\txC}\| = O(\beta d \gamma^{-d} \epsilon).$ The proof for the latter builds upon two key facts: (a) $\|(P^{\pib})^k - (\hat{P}^{\pib})^k\| \leq \sum_{h = 1}^k \|\h{P}^{\pib}\|^{h - 1} \|\h{P}^{\pib} - P^{\pib} \| \|p^{\pib}\|^{k - h} =   O(k \epsilon)$ for any $k \geq 0$, and (b)  $|e^x - 1| = O(x)$ as $x \to 0.$ Next, we decompose the LHS of \eqref{eq: bias bound} to get
$$
 \sum_{s} \left( \prod_{i = 1}^4 X_i(s) - \prod_{i = 1}^4 \h{X}_i(s) \right) 
       = 
       \sum_s \sum_{i = 1}^4 \h{X}_1(s) \cdots \h{X}_{i - 1}(s) \left(X_i(s) - \h{X}_i(s)\right) \times 
  X_{i + 1}(s) \cdots X_4(s),
$$
where $X_1(s) = d_{\pid^{\txC}}(s) \in \bR,$ $X_2(s) = (\nabla_\theta\log\pid^{\txC}(\cdot|s))^\tr \in \bR^{S \times A},$ $X_3(s) =  D(\pid^{\txC}(\cdot|s)) \in \bR^{A \times A},$ $X_4(s) = Q^{\pid^{\txC}}(s, \cdot) \in \bR^{A \times A},$ and $\h{X}_1(s), \ldots, \h{X}_4(s)$ are similarly defined with $\pid^{\txC}$ replaced by $\pidhat^{\txC}.$ Then, we show that, for $i = 1, \ldots, 4,$ (i) $\|X_i(s) - \h{X}_i(s)\| = O(\gamma^{-d} \epsilon)$ and (ii) $\max\{\|X_i\|,\|\h{X}_i\|\}$ is bounded by problem parameters. From this, the desired result follows.
\end{proof}

\begin{proof}
We have
\begin{align}
    \frac{\partial}{\partial \theta} & \left(\nu^\top V^{\pi_\theta}\right) -   \frac{\partial}{\partial \theta}  \left(\nu^\top V^{\pi'_\theta}\right) \\
  = {} & \bE_{s\sim{d_{\pi_\theta}},a\sim \pi_\theta(\cdot|s)} \left[\nabla_\theta\log\pi_\theta(a|s)Q^{\pi_\theta}(s,a)\right] - \bE_{s\sim{d_{\h{\pi}_\theta}}, 
 a\sim\h{\pi}_\theta(\cdot|s)}\left[\nabla_\theta\log \h{\pi}_\theta(a|s)Q^{\h{\pi}_\theta}(s,a)\right] \\
 = {} & \sum_{s, a} \left(d_{\pi_\theta}(s) \pi_\theta(a|s) \nabla_\theta\log\pi_\theta(a|s)Q^{\pi_\theta}(s,a) -  d_{\h{\pi}_\theta}(s)  \h{\pi}_\theta(a|s) \nabla_\theta \log\h{\pi}_\theta(a|s)Q^{\h{\pi}_\theta}(s,a) \right) \\
  % = {} & \ones_{S\cdot A}^\top D(d_{\pi_\theta})D(\pi_\theta) D(Q) \nabla_\theta\log\pi_\theta \mbox{ where }\nabla_\theta\log\pi_\theta \in \bR^{S\cdot A \times S} \\
  %
  = {} & \sum_{s} \Big(d_{\pi_{\theta}}(s) (\nabla_\theta\log\pi_\theta(\cdot|s))^\tr D(\pi_\theta(\cdot|s)) Q^{\pi_\theta}(s, \cdot) \\
  {} & - d_{\h{\pi}_{\theta}}(s) (\nabla_\theta\log\h{\pi}_\theta(\cdot|s))^\tr D(\h{\pi}_\theta(\cdot|s)) Q^{\h{\pi}_\theta}(s, \cdot)\Big)\\
 = {} & \sum_{s} \left( \prod_{i = 1}^4 X_i(s) - \prod_{i = 1}^4 \h{X}_i(s) \right) 
     \\
 = {} & \sum_s \sum_{i = 1}^4 \h{X}_1(s) \cdots \h{X}_{i - 1}(s) \left(X_i(s) - \h{X}_i(s)\right) X_{i + 1}(s) \cdots X_4(s),
 %
 % = {} & F d_{\pi_\theta} - \h{F} d_{\h{\pi}_\theta},
\end{align}
where $X_1(s) = d_{\pi_\theta}(s) \in \bR,$ $X_2(s) = (\nabla_\theta\log\pi_\theta(\cdot|s))^\tr \in \bR^{S \times A},$ $X_3(s) =  D(\pi_\theta(\cdot|s)) \in \bR^{A \times A},$ $X_4(s) = Q^{\pi_\theta}(s, \cdot) \in \bR^{A \times A},$ and $\h{X}_1(s), \ldots, \h{X}_4(s)$ are similarly defined with $\pi_\theta$ replaced by $\h{\pi}_\theta.$

Therefore,
\begin{equation}
    \left\|\frac{\partial}{\partial \theta} \left(\nu^\top V^{\pi_\theta}\right) -   \frac{\partial}{\partial \theta}  \left(\nu^\top V^{\pi'_\theta}\right)\right\| \leq \left(\max_s \Gamma(s) \right) S,
\end{equation}
where 
%
%
% \begin{equation}
%     \Gamma(s) = \|\sum_s \sum_{i = 1}^4 \h{X}_1(s) \cdots \h{X}_{i - 1}(s) \left(X_i(s) - \h{X}_i(s)\right) X_{i + 1}(s) \cdots X_4(s)\|.
% \end{equation}
% Aug 6 2024 update by GalD: I'm dropping the sum on s since I think it's a typo
\begin{equation}
    \Gamma(s) = \|\sum_{i = 1}^4 \h{X}_1(s) \cdots \h{X}_{i - 1}(s) \left(X_i(s) - \h{X}_i(s)\right) X_{i + 1}(s) \cdots X_4(s)\|.
\end{equation}

Next, since $d_{\pi_\theta}, d_{\hat{\pi}_\theta}, \pi_\theta,$ and $\h{\pi}_\theta$ are all distributions, we have
\begin{equation}
    \max\{|X_1(s)|, |\h{X_1}(s)|, |X_3(s, a)|, |\h{X_3}(s, a)|\} \leq 1.
\end{equation}
Separately, using Lemma~\ref{lem:analytic_gradC}, 
we have
\begin{equation}
    \|X_2\| = \|\nabla_\theta \log\pi_\theta(a|s) \| \leq \beta  (\|I_A\| + \|\ones_A \pi_\theta^\top\|) \|P_s\| \|(P^{\pib})^{d - 1}\|.
\end{equation}
Since all rows of the above matrices have non-negative entries that add up to $1,$ we get
\begin{equation} \label{eq: C bound}
    \|Y\| \leq 2 \beta .
\end{equation}

In the rest of the proof, we bound each of $\|X_1 - \h{X_1}\|, \ldots, \|X_4 - \h{X_4}\|.$

Finally, 
\begin{equation}
    \|X_4\| \leq \frac{1}{1 - \gamma}.
\end{equation}
Similarly, the same bounds hold for $\h{X_1},\h{X_2},\h{X_3}$ and $\h{X_4}.$

 From, we have
\begin{align}
    \|X_1 - \h{X_1}\| \leq {} & (1 - \gamma) \sum_{t = 0}^\infty \gamma^t \|\nu^\top (P^{\pi_\theta})^t - \nu^\top (P^{\h{\pi}_\theta})^t\| \\
    \leq {} & (1 - \gamma) \|\nu\| \sum_{t = 0} \gamma^t td \epsilon  \label{e:Diff.P.Bd} \\
    \leq {} & (1 - \gamma) d \epsilon \sum_{t = 0}^\infty \gamma^t t \\
    = {} &  \frac{\gamma d \epsilon}{1 - \gamma}.
\end{align}
%
% \gug{Check the constant in front of $\epsilon$ in the above derivation.Also, we may have to derive the proof of \eqref{e:Diff.P.Bd}. Note the presence of $\h{}$ in the expression.} 

The last relation follows from the fact that $(1 - \gamma)^{-1} = \sum_{t = 0}^\infty \gamma^t,$ which in turn implies
\begin{equation}
     \gamma \frac{\partial }{\partial \gamma } \left(\frac{1}{1 - \gamma}\right) = \sum_{t = 0}^\infty t \gamma^t.
\end{equation}

From Lemma~\ref{lem:pol.Diff.Bd}, it follows that 
\begin{equation}
    \|X_3 - \h{X_3}\| = O(\beta d \epsilon).
\end{equation}

Next, recall that from Lemma~\ref{lem:analytic_gradC} that
\begin{align*} 
 X_2(s,\cdot) = \beta\left[I_{A} -  \ones_A {(\pi_\theta})^\top \right]P_s \left(P^{\pib}\right)^{d-1}.
\end{align*}
Then,
\begin{align}
    \|X_2(s,\cdot) - \h{X_2}(s,\cdot)\| \leq &  \|\beta\left[I_{A} -  \ones_A {(\pi_\theta})^\top \right]P_s\|\| \left(P^{\pib}\right)^{d-1} - \left(\h{P}^{\pib}\right)^{d-1}\| \label{eq: first term}\\
    {} & +  \|  \beta\left[I_{A} -  \ones_A {(\pi_\theta})^\top \right]\| \|P_s - \h{P}_s\| \|\left(\h{P}^{\pib}\right)^{d-1}\| \label{eq: second term} \\
    {} & +   \beta \|\ones_A {(\pi_\theta})^\top  - \ones_A {(\h{\pi}_\theta})^\top \| \|\h{P}_s \left(\h{P}^{\pib}\right)^{d-1}\|\label{eq: third term}.
    % = & O(\beta\gamma^d d \epsilon).
\end{align}
Following the same argument as in \eqref{eq: C bound} and applying Lemma~\ref{lem: power decomposition}, we have that \eqref{eq: first term} is $O(\beta d \epsilon).$ Similarly, from the argument of \eqref{eq: C bound}, Eq.~\eqref{eq: second term} is $O(\beta \epsilon)$. Lastly, \eqref{eq: third term} is $O(\beta d \epsilon)$ due to Lemma~\ref{lem:pol.Diff.Bd}. Putting the above three terms together, we have that 
\begin{equation}
    \|X_2(s,\cdot) - \h{X_2}(s,\cdot)\| = O(\beta d \epsilon).
\end{equation}

Since the state-action value function satisfies the Bellman equation, we have
\begin{equation}
    Q^{\pi_\theta} = r + \gamma P Q^{\pi_\theta}
\end{equation}
and
\begin{equation}
    Q^{\h{\pi}_\theta} = \h{r} + \gamma \h{P} Q^{\h{\pi}_\theta}.
\end{equation}
Consequently,
\begin{align}
    \|Q^{\pi_\theta} - Q^{\h{\pi}_\theta}\| \leq {} & \|r - \h{r}\| + \gamma \|P Q^{\pi_\theta} - P  Q^{\h{\pi}_\theta}\| + \gamma \|P Q^{\h{\pi}_\theta} - \h{P} Q^{\h{\pi}_\theta}\| \\
    \leq {} & \epsilon + \gamma \|P\| \|Q^{\pi_\theta} - Q^{\h{\pi}_\theta}\| + \gamma \|P - \h{P}\| \|Q^{\h{\pi}_\theta}\| \\
    \leq {} & \epsilon + \gamma \|Q^{\pi_\theta} - Q^{\h{\pi}_\theta}\| + \frac{\gamma}{1 - \gamma} \epsilon,
\end{align}
which finally shows that
\begin{equation}
    \|X_4 - \h{X_4}\| = \|Q^{\pi_\theta} - Q^{\h{\pi}_\theta}\| \leq \frac{\epsilon}{(1 - \gamma)^2}.
\end{equation}

\end{proof}

\section{Experiments} \label{app:experiments}

% We conduct our experiments on multiple games from the Atari simulation suite \citep{bellemare2013arcade}. As a baseline, we train a PPO \citep{schulman2017proximal} agent with $256$ workers in parallel. In a hyperparameter search, we found this number of workers to be the best in terms of run-time.
\subsection{Implementation Details}
The environment engine is the highly efficient Atari-CuLE \citep{dalton2020accelerating}, a CUDA-based version of Atari that runs on GPU. Similarly, we use Atari-CuLE for the GPU-based breadth-first TS as done in \cite{dalal2021improve}.

We train \treepol{} for depths $d=1 \dots 8,$ with a single worker. We use five seeds for each experiment. 

For the implementation, we extend Stable-Baselines3 \citep{raffin2019stable} with all parameters taken as default from the original PPO paper \citep{schulman2017proximal}. For depths $d \geq 3$, we limited the tree to a maximum width of $1024$ nodes and pruned non-promising trajectories in terms of estimated weights. Since the distributed PPO baseline advances significantly faster in terms of environment steps, for a fair comparison, we ran all experiments for one week on the same machine and use the wall-clock time as the x-axis. We use Intel(R) Xeon(R) CPU E5-2698 v4 @ 2.20GHz equipped with one NVIDIA Tesla V100 32GB.  

\subsection{GPU-Based Tree Expansion Implementation}
\label{app:tree_expansion}

Figure~\ref{fig:tree_expansion} illustrates the GPU-based tree expansion mechanism used in our implementation. We achieve efficient parallelization by duplicating and concatenating all states in the current level of the tree with each possible action, then advancing them simultaneously with a single forward pass through the simulator.

\begin{figure}[h]
\centering
\includegraphics[width=0.9\textwidth]{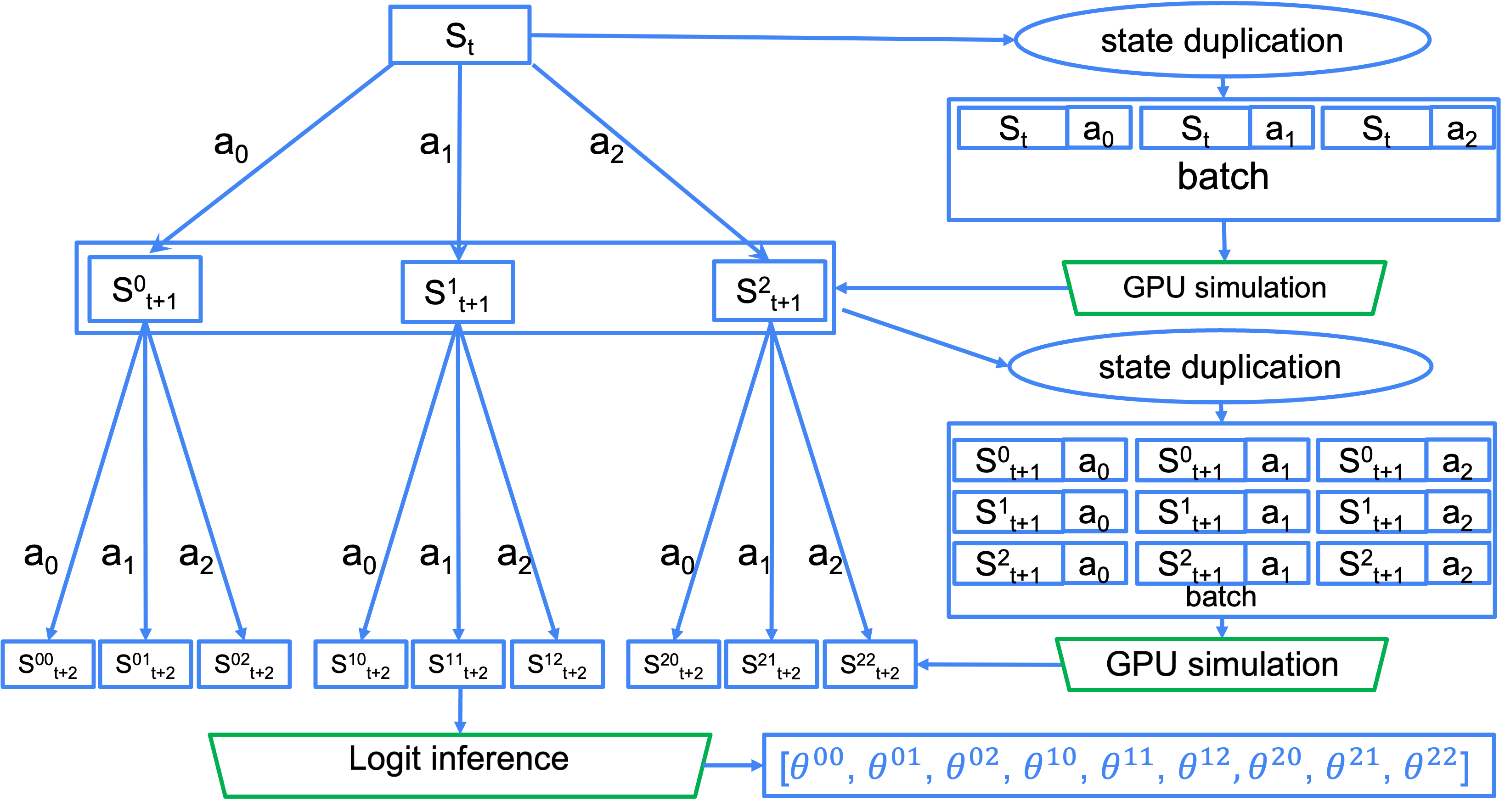}
\caption{A diagram of the tree expansion used by SoftTreeMax. In every step, the states in the current level of the tree are duplicated and concatenated with each possible action. The resulting state-action pairs are then fed as a batch to the GPU simulator to generate the next level of states. Finally, the states of the last level $d$ are inserted into the neural network $W_\theta$ and the logits are computed using the corresponding rewards along each trajectory.}
\label{fig:tree_expansion}
\end{figure}

The implementation follows these steps:
\begin{enumerate}
    \item Start with a root state $s_0$ and expand it across all actions $a \in \mathcal{A}$.
    \item For each depth level $t$ from 1 to $d$:
        \begin{itemize}
            \item Collect all states $s_t$ from the previous level.
            \item Duplicate each state for each action to create a batch of state-action pairs.
            \item Submit the entire batch to the GPU simulator in a single forward pass.
            \item Collect the resulting next states and corresponding rewards.
            \item If $t < d$ and pruning is enabled, select the top-$k$ most promising branches based on accumulated rewards.
        \end{itemize}
    \item For the leaf states (at depth $d$), compute the neural network outputs $W_\theta(s_d)$.
    \item Combine the accumulated rewards along each trajectory with the corresponding leaf state values.
    \item Compute the final policy logits according to Equation~\eqref{eq:logit} and apply the softmax operation.
\end{enumerate}

This parallel implementation allows us to efficiently explore a larger number of trajectories compared to sequential tree expansion, making deeper tree depths practically feasible.

\subsection{Algorithms}
\label{app:algorithms}

This section provides the pseudocode for our SoftTreeMax implementation. Algorithm~\ref{alg:SoftTreeMax} details the C-SoftTreeMax policy computation, which efficiently utilizes GPU parallelization to perform tree expansion. Algorithm~\ref{alg:SoftTreeMaxPPO} shows how SoftTreeMax integrates with the PPO algorithm, distinguishing the usage of our new policy in red.

{\small{
\begin{algorithm}[tb] 
   \caption{C-SoftTreeMax}
   \label{alg:SoftTreeMax}
\begin{algorithmic}
   \STATE {\bfseries Input:} GPU environment ${\mathcal G}$, network $\theta$, depth $d$
   \STATE {\bfseries Init tensors:} state $\bar{S} = [s],$ action $\bar{A_0} = \left[0, 1, 2, .., A-1\right]$, reward  $\bar{R}=[0]$
   \FOR{$i_d=0$ {\bfseries to} $d-1$}
   \STATE $\bar{S} \leftarrow \bar{S} \times A,~~\bar{R} \leftarrow \bar{R} \times A$ \COMMENT{Replicate state and reward tensors $A$ times} 
   \STATE $\bar{r},\bar{S}' = {\mathcal G}([\bar{S},\bar{A}])$ \COMMENT{Feed  $[\bar{S},\bar{A}]$ to simulator and advance}
   \STATE $\bar{R} \leftarrow \bar{R} + \gamma^{i_d} \bar{r}, ~~\bar{S} \leftarrow \bar{S}'$ \COMMENT{Accumulate discounted reward }
   \STATE  $\bar{A} \leftarrow \bar{A} \times A$ \COMMENT{Replicate action tensor $A$ times}
   \ENDFOR
   \STATE $l_{s,a} \leftarrow \text{Average}^{\pi_b}_{A_0=a}(\bar{R} + \gamma^d \theta( \bar{S}))/\gamma^d$\COMMENT{Weighted average induced by $\pi_b$}
   \STATE \textbf{Return} $\pi (a|s_0)  \propto \exp \left[\beta l_{s, a}(d;\theta) \right]$ \COMMENT{Return optimal action at the root}
\end{algorithmic}
\end{algorithm}
}}

\begin{algorithm}
\caption{SoftTreeMax-PPO}
\label{alg:SoftTreeMaxPPO}
\begin{algorithmic}[1]
\STATE Initialize policy parameters $\theta_0$
\STATE Initialize value function parameters $\phi$
\FOR{$k = 0, 1, 2, \ldots$}
    \STATE Collect set of trajectories $\mathcal{D}_k = \{\tau_i\}$ by running policy {\color{red}$\pi_{d,\theta_k}$} from Algorithm~\ref{alg:SoftTreeMax}
    \STATE Compute rewards-to-go $\hat{R}_t$
    \STATE Compute advantage estimates $\hat{A}_t$ using GAE with $\lambda=0.95$
    \FOR{each epoch}
        \FOR{each minibatch}
            \STATE Compute policy ratio $r_t(\theta) = \frac{{\color{red}\pi_{d,\theta}(a_t|s_t)}}{{\color{red}\pi_{d,\theta_k}(a_t|s_t)}}$
            \STATE Compute clipped surrogate objective:
            \STATE $L^{CLIP}(\theta) = \mathbb{E}_t[\min(r_t(\theta)\hat{A}_t, \text{clip}(r_t(\theta), 1-\epsilon, 1+\epsilon)\hat{A}_t)]$
            \STATE Update $\theta$ with gradient step on $L^{CLIP}(\theta)$
            \STATE Compute value function loss: $L^{VF}(\phi) = (V_\phi(s_t) - \hat{R}_t)^2$
            \STATE Update $\phi$ with gradient step on $L^{VF}(\phi)$
        \ENDFOR
    \ENDFOR
\ENDFOR
\end{algorithmic}
\end{algorithm}

Note that in Algorithm~\ref{alg:SoftTreeMaxPPO}, we use Generalized Advantage Estimation (GAE) with $\lambda=0.95$ for calculating advantage estimates, which is the standard configuration in the stable-baselines3 PPO implementation that we build upon.

\subsection{Time-Based Training Curves}
We provide the training curves in Figure~\ref{fig:train_curves}. For brevity, we exclude a few of the depths from the plots. As seen, there is a clear benefit for \treepol{} over distributed PPO with the standard softmax policy. In most games, PPO with the \treepol{} policy shows very high sample efficiency: it achieves higher episodic reward although it observes much less episodes, for the same running time. 

\begin{figure}[H]
    \centering
\includegraphics[width=1.00\textwidth]{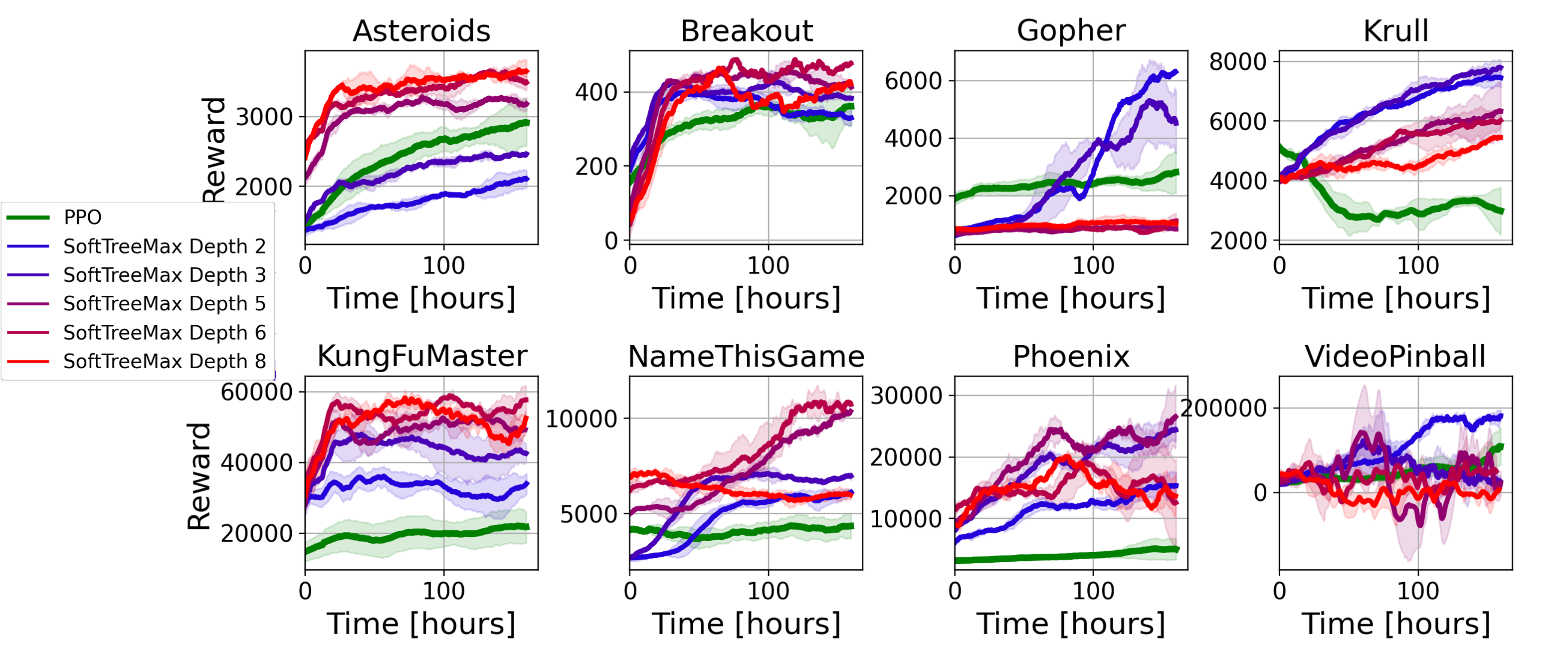}
    \caption{\textbf{Training curves: GPU \treepol{} (single worker) vs PPO ($\bf{256}$ GPU workers).} The plots show average reward and standard deviation over 5 seeds.  The x-axis is the wall-clock time. The runs ended after one week with varying number of time-steps. The training curves correspond to the evaluation runs in Figure~\ref{fig:variance_curves}.}  \label{fig:train_curves}
\end{figure}

\subsection{Step-Based Training Curves} \label{sec: step based plots}
In Figure~\ref{fig:train_curves_steps} we also provide the same convergence plots where the x-axis is now the number of online interactions with the environment, thus excluding the tree expansion complexity. As seen, due to the complexity of the tree expansion, less steps are conducted during training (limited to one week) as the depth increases. In this plot, the monotone improvement of the reward with increasing tree depth is noticeable in most games. 

\begin{figure}[H]
    \centering
\includegraphics[scale=0.53]{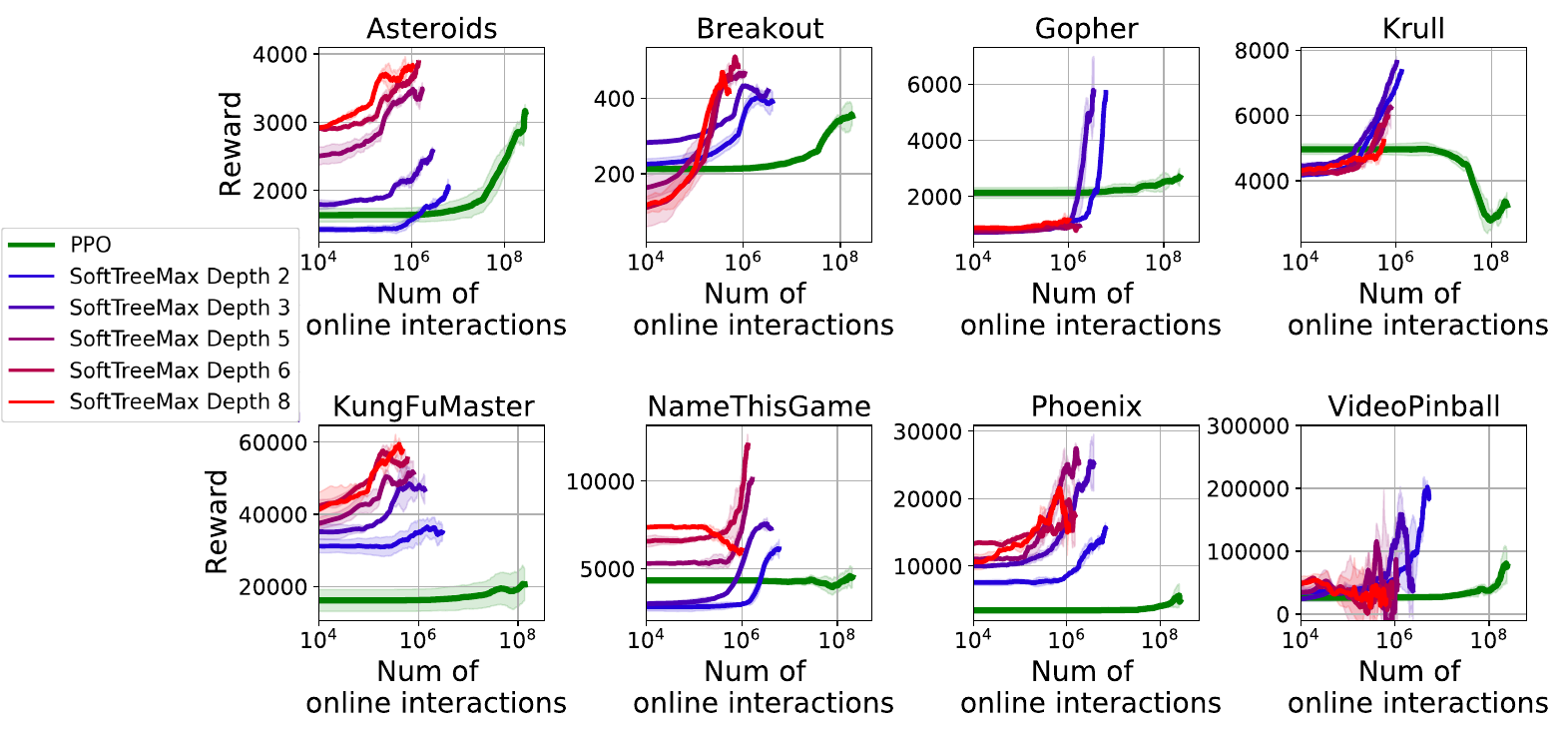}
    \caption{\textbf{Training curves: GPU \treepol{} (single worker) vs PPO ($\bf{256}$ GPU workers).} The plots show average reward and standard deviation over 5 seeds. The x-axis is the number of online interactions with the environment. The runs ended after one week with varying number of time-steps. The training curves correspond to the evaluation runs in Figure~\ref{fig:variance_curves}.}  \label{fig:train_curves_steps}
\end{figure}

We note that not for all games we see monotonicity. Our explanation for this phenomenon relates to how immediate reward contributes to performance compared to the value. Different games benefit differently from long-term as opposed to short-term planning. Games that require longer-term planning need a better value estimate. A good value estimate takes longer to obtain with larger depths, in which we apply the network to states that are very different from the ones observed so far in the buffer (recall that as in any deep RL algorithm, we train the model only on states in the buffer). If the model hasn’t learned a good enough value function yet, and there is no guiding dense reward along the trajectory, the policy becomes noisier, and can take more steps to converge -- even more than those we run in our week-long experiment. 

For a concrete example, let us compare Breakout to Gopher. Inspecting Fig.~\ref{fig:train_curves_steps}, we observe that Breakout quickly (and monotonically) gains from large depths since it relies on the short term goal of simply keeping the paddle below the moving ball. In Gopher, however, for large depths (>=5), learning barely started even by the end of the training run. Presumably, this is because the task in Gopher involves multiple considerations and steps: the agent needs to move to the right spot and then hit the mallet the right amount of times, while balancing different locations. This task requires long-term planning and thus depends more strongly on the accuracy of the value function estimate. In that case, for depth 5 or more, we would require more train steps for the value to “kick in” and become beneficial beyond the gain from the reward in the tree.

The figures above convey two key observations that occur for at least some non-zero depth: (1) The final performance with the tree is better than PPO (Fig. \ref{fig:variance_curves}); and (2) the intermediate step-based results with the tree are better than PPO (Fig.~\ref{fig:train_curves_steps}). This leads to our main takeaway from this work –- there is no reason to believe that the vanilla policy gradient algorithm should be better than a multi-step variant. Indeed, we show that this is not the case.

\section{Further discussion}

\subsection{The case of $\lambda_2(P^{\pi_b})=0$} \label{sec: zero grad}
When $P^{\pi_b}$ is rank one, it is not only its variance that becomes $0$, but also the norm of the gradient itself (similarly to the case of $d\rightarrow\infty$). Note that such a situation will happen rarely, in degenerate MDPs. This is a local minimum for \treepol{} and it would cause the PG iteration to get stuck, and to the optimum in the (desired but impractical) case where $\pi_b$ is the optimal policy. However, a similar phenomenon was also discovered in the standard softmax with deterministic policies: $\theta(s, a) \rightarrow \infty$ for one $a$ per $s$. PG with softmax would suffer very slow convergence near these local equilibria, as observed in \cite{mei2020escaping}. To see this, note that the softmax gradient is $\nabla_\theta\log \pi_\theta(a|s) = e_{a} - \pi_\theta(\cdot|s),$ where $e_{a} \in [0,1]^A$ is the vector with 0 everywhere except for the $a$-th coordinate. I.e., it will be zero for a deterministic policy. \treepol{} avoids these local optima by integrating the reward into the policy itself (but may get stuck in another, as discussed above).